\documentclass{siamart}
\usepackage{amssymb,graphicx,enumerate,url,mathtools}
\usepackage[mathscr]{eucal}
\usepackage{color}
\usepackage{tikz,tikz-cd}
\usepackage[normalem]{ulem}
\usepackage{soul}
\usepackage{url}

\usepackage{pgf}
\usepackage{pgffor}
\usepgfmodule{shapes}
\usepgfmodule{plot}
\usetikzlibrary{decorations}
\usetikzlibrary{arrows}
\usetikzlibrary{snakes}

\newcommand{\physicalpendulum}{
    \pgfpathmoveto{\pgfpoint{0cm}{2cm}}
    \pgfpathcurveto{\pgfpoint{.2cm}{2cm}}{\pgfpoint{.2cm}{2cm}}{\pgfpoint{1cm}{1cm}}
    \pgfpathcurveto{\pgfpoint{1.75cm}{0cm}}{\pgfpoint{1.8cm}{0cm}}{\pgfpoint{2cm}{-2cm}}
    \pgfpathcurveto{\pgfpoint{2.2cm}{-4cm}}{\pgfpoint{2.3cm}{-4.2cm}}{\pgfpoint{2.5cm}{-5cm}}
    \pgfpathcurveto{\pgfpoint{2.7cm}{-5.8cm}}{\pgfpoint{3cm}{-6cm}}{\pgfpoint{3cm}{-7cm}}
    \pgfpathcurveto{\pgfpoint{3cm}{-8cm}}{\pgfpoint{2.25cm}{-8.75cm}}{\pgfpoint{2cm}{-9cm}}
    \pgfpathcurveto{\pgfpoint{1.75cm}{-9.25cm}}{\pgfpoint{1cm}{-10cm}}{\pgfpoint{0cm}{-10cm}}
    \pgfpathcurveto{\pgfpoint{-2cm}{-10cm}}{\pgfpoint{-3cm}{-8cm}}{\pgfpoint{-3cm}{-7cm}}
    \pgfpathcurveto{\pgfpoint{-3cm}{-3cm}}{\pgfpoint{-1cm}{-2cm}}{\pgfpoint{-1cm}{-1cm}}
    \pgfpathcurveto{\pgfpoint{-1cm}{0cm}}{\pgfpoint{-.5cm}{2cm}}{\pgfpoint{0cm}{2cm}}
    \pgfusepath{fill,stroke}
}

\makeatletter
\newcommand*{\saved@uline}{}
\let\saved@uline\uline

\newcommand*{\mathuline}{%
  \mathpalette{\math@uline\saved@uline}%
}
\newcommand*{\math@uline}[3]{%
  \mbox{#1{$#2#3\m@th$}}%
}

\renewcommand*{\uline}{%
  \relax  
  \ifmmode
    \expandafter\mathuline
  \else
    \expandafter\saved@uline
  \fi
}
\makeatother

\DeclareMathOperator{\SO}{SO}
\DeclareMathOperator{\U}{U}

\headers{Cohomology of Cryo-Electron Microscopy}{K.~Ye and L.-H.~Lim}

\title{{Cohomology of Cryo-Electron Microscopy}\thanks{LHL and KY are generously supported by AFOSR FA9550-13-1-0133, DARPA D15AP00109, NSF IIS 1546413, DMS 1209136, and DMS 1057064.}}

\author{
  Ke~Ye\thanks{Computational and Applied Mathematics Initiative, Department of Statistics, University of Chicago, 5747 South Ellis Avenue, Chicago, IL 60637, USA, (\email{kye@galton.uchicago.edu}).}
  \and
  Lek-Heng~Lim\thanks{Computational and Applied Mathematics Initiative, Department of Statistics, University of Chicago, 5747 South Ellis Avenue, Chicago, IL 60637, USA,
    (corresponding author: \email{lekheng@galton.uchicago.edu}).}
}

\begin{document}
\maketitle

\begin{abstract}
The goal of cryo-electron microscopy (EM) is to reconstruct the $3$-dimensional structure of a molecule from a collection of its $2$-dimensional projected images. In this article, we show that the basic premise of cryo-EM --- patching together $2$-dimensional projections to reconstruct a $3$-dimensional object --- is naturally  one of \v{C}ech cohomology with $\SO(2)$-coefficients. We deduce that every cryo-EM reconstruction problem corresponds to an oriented circle bundle on a simplicial complex, allowing us to classify cryo-EM problems via principal bundles.  In practice, the $2$-dimensional images are noisy and a main task in cryo-EM is to denoise them. We will see how the aforementioned insights can be used towards this end.
\end{abstract}



\section{Introduction}\label{sec:intro}

The problem of cryo-electron microscopy (cryo-EM) asks for the following: Given a collection of noisy $2$-dimensional ($2$D) projected images, reconstruct the $3$-dimensional ($3$D) structure of the molecule that gave rise to these images. Viewed from a high level, it takes the form of an inverse problem similar to those in medical imaging \cite{medical1,medical2,medical3}, remote sensing \cite{remote1,remote2}, or underwater acoustics \cite{underwater1,underwater2}, except that for cryo-EM the data comes from an electron microscope instead of a CT scanner, radar, or sonar. However, when examined at a finer level of detail, one realizes that the cryo-EM problem possesses mathematical structures that are quite different from those of other classical inverse problems. It has inspired studies from the perspectives of representation theory \cite{HS1, HS2}, differential geometry \cite{SW1, SW2}, and is related to profound problems in computational complexity \cite{BCS} and operator theory \cite{BKS}. This article examines the problem from an algebraic topological angle --- we will show that the problem of cryo-EM is a problem of cohomology, or, more specifically, the {\v C}ech cohomology of a simplicial complex with coefficients in the Lie group $\SO(2)$ and the discrete group $\SO(2)_d$, i.e., $\SO(2)$ endowed with the discrete topology.

Despite its abstract appearance, the aforementioned cohomology framework is actually concrete and natural. The fact that cohomology has an important role to play in understanding $2$D projections of $3$D objects is already evident in simple examples like the  Penrose tribar or Escher brick, as we will see in Section~\ref{sec:tribar}. Our analysis of discrete and continuous cryo-EM cocycles requires a more sophisticated type of cohomology but is essentially along the same lines. In fact, the same ideas  that we use to study the cryo-EM problem also underlies the classical field theory of electromagnetism \cite{Chern}.  The cohomology framework allows us to classify cryo-EM cocycles: Given two different collections of $2$D projected images, are they equivalent in the sense that they will give us the same $3$D reconstruction? The insights gained also shed light on the denoising techniques: What are we really trying to achieve when we minimize a certain loss function to denoise cryo-EM images?

The technique of cryo-electron microscopy has been described in great detail in \cite{Frank1, Frank2} and more than adequately summarized in \cite{HS1, HS2,  PZF, SS, SW1, SW2, SZSH, VG, Heel}. It suffices to provide a very brief review here. A more precise mathematical model, for the following high-level description will be given in Section~\ref{sec:main}. The basic idea is that one first immobilizes many identical copies of a molecule in ice and employs an electron microscope to produce $2$D images of the molecule. As each copy of the molecule is frozen in some unknown orientation, each of the $2$D images may be regarded as a projection of the molecule from an unknown viewing direction. The cryo-EM dataset is then the set of these $2$D  projected images. Such a $2$D image shows not only the shape of the molecule in the plane of the viewing direction but also contains information about the density of the molecule, captured in the intensity of each pixel of the $2$D image  \cite{Natterer}. The ultimate goal of cryo-EM is to construct the $3$D structure of the molecule from a cryo-EM dataset. In practice, these $2$D images are very noisy due to various issues ranging from the electron dosage of the microscope to the structure of the ice in which the molecule are frozen. Hence the main difficulty in  cryo-EM reconstruction is to denoise these $2$D images by determining the true viewing directions of these noisy $2$D images so that one may take averages of nearby images.  There has been much significant progress toward this goal in recent years \cite{ PZF, SS, SZSH, VG, Heel}. 

Our article attempts to understand cryo-EM datasets of $2$D images via {\v C}ech and singular cohomology groups. We will see that for a given molecule, the information extracted from its $2$D cryo-EM images determines a cohomology class of a two-dimensional simplicial complex. Furthermore, each of these cohomology classes corresponds to an oriented circle bundle on this simplicial complex. We note that there are essentially two interpretations of cohomology: obstruction and moduli. On the one hand, a cohomology group quantifies the \textit{obstruction} from local to global. For example, this is the sense in which cohomology is used when demonstrating the non-existence of an impossible figure \cite{Penrose} or in the solution of the Mittag-Leffler problem \cite[p.~34]{GH}. On the other hand, a cohomology group may also be used to describe a collection of mathematical objects, i.e., it serves as a \textit{moduli} space for these objects. For example, when we use a cohomology group to parameterize all divisors or all line bundles on an algebraic variety \cite[p.~143]{Hartshorne}, it is used in this latter sense.

The line bundles example is a special case of a more general statement: A cohomology group serves as the moduli space of principal bundles over a topological space. This forms the basis for our use of cohomology in the cryo-EM reconstruction problem --- as a moduli space for all possible cryo-EM datasets. Obviously, such a classification of cryo-EM datasets comes under the implicit assumption that the $2$D images in a dataset are noise-free.  Our classification depends on a standard mathematical \emph{model} for molecules in the context of cryo-electron microscopy under a noise-free assumption. Here the reader is reminded that a molecule is a physical notion and not a mathematical one. A mathematical answer to the question `What is a molecule?' depends on the context. In one theory, a molecule may be a solution to a Schr\"odinger \textsc{pde} (e.g., quantum chemistry) whereas in another, it may be a path in a $6N$-dimensional phase space (e.g., molecular dynamics). In our model, a molecule is a real-valued function on $\mathbb{R}^3$ representing potential. When our images are noisy, this model gives us a natural way, namely, the cocycle condition, to denoise them by fitting them to the model. Various methods for denoising cryo-EM images \cite{SS, SZSH} may be viewed as nonlinear regression for fitting the cocycle condition under additional assumptions.

\section{Cohomology and $2$D projections of $3$D objects}\label{sec:tribar}

The idea that cohomology arises whenever one attempts to analyze $2$D projections of $3$D objects was first pointed out by Roger Penrose, who proposed in \cite{Penrose}  a cohomological argument to analyze Escher-type optical illusions. In the following, we present Penrose's elegantly simple example since it illustrates some of the same principles that underly our more complicated use of cohomology in cryo-EM.

We follow the spirit of Penrose's arguments in \cite{Penrose} but we will deviate slightly to be more in-line with our discussions of cryo-EM and to obtain a proof for the nonexistence of Penrose tribar. The few unavoidable topological jargons are defined in Section~\ref{sec:three} but they are used in such a way that one could grasp the intuitive ideas involved even without knowledge of the jargons. To be clear, a $3$D object is one that can be embedded in $\mathbb{R}^3$ by an injective map $J$ such that $J(ax+by) = a J(x) + b J(y)$ whenever $x$, $y$, $ax + by $ are points in this object, and $a, b \in \mathbb{R}$.

The Penrose \textit{tribar} is defined to be a fictitious $3$D object --- fictitious as it does not exist in $\mathbb{R}^3$ ---  obtained by gluing three rectangular solid cuboids (i.e., bars) $L_1, L_2, L_3$ in $\mathbb{R}^3$ as follows: $L_i$ is glued to $L_j$ by identifying a cubical portion $L_{ij}$ at one end of $L_i$ with a cubical portion $L_{ji}$ at one end of $L_j$ as depicted in Figure~\ref{fig:tribar}(b), $i,j=1,2,3$.
\begin{figure}[h]
\centering
    \begin{tikzpicture}[scale=0.5, line join=bevel]
	
    \pgfmathsetmacro{\a}{2.5}
    \pgfmathsetmacro{\b}{0.9}

    \tikzset{%
      apply style/.code     = {\tikzset{#1}},
      triangle_edges/.style = {thick,draw=black}
    }

    \foreach \theta/\facestyle in {%
        0/{triangle_edges, fill = gray!50},
      120/{triangle_edges, fill = gray!25},
      240/{triangle_edges, fill = gray!90}%
    }{
      \begin{scope}[rotate=\theta]
        \draw[apply style/.expand once=\facestyle]
          ({-sqrt(3)/2*\a},{-0.5*\a})                     --
          ++(-\b,0)                                       --
            ({0.5*\b},{\a+3*sqrt(3)/2*\b})                -- 
            ({sqrt(3)/2*\a+2.5*\b},{-.5*\a-sqrt(3)/2*\b}) -- 
          ++({-.5*\b},-{sqrt(3)/2*\b})                    -- 
            ({0.5*\b},{\a+sqrt(3)/2*\b})                  --
          cycle;
        \end{scope}
      }	
      
	  \end{tikzpicture}
\qquad
\begin{tikzpicture}
\pgfmathsetmacro{\cubex}{3}
\pgfmathsetmacro{\cubey}{0.5}
\pgfmathsetmacro{\cubez}{0.5}
\draw[black,fill=black!20] (0,0,0) -- ++(-\cubex,0,0) -- ++(0,-\cubey,0) -- ++(\cubex,0,0) -- cycle;
\draw[black,fill=black!20] (0,0,0) -- ++(0,0,-\cubez) -- ++(0,-\cubey,0) -- ++(0,0,\cubez) -- cycle;
\draw[black,fill=black!20] (0,0,0) -- ++(-\cubex,0,0) -- ++(0,0,-\cubez) -- ++(\cubex,0,0) -- cycle;

\pgfmathsetmacro{\cubex}{0.5}
\pgfmathsetmacro{\cubey}{0.5}
\pgfmathsetmacro{\cubez}{0.5}
\draw[black,fill=red!20] (0,0,0) -- ++(-\cubex,0,0) -- ++(0,-\cubey,0) -- ++(\cubex,0,0) -- cycle;
\draw[black,fill=red!20] (0,0,0) -- ++(0,0,-\cubez) -- ++(0,-\cubey,0) -- ++(0,0,\cubez) -- cycle;
\draw[black,fill=red!20] (0,0,0) -- ++(-\cubex,0,0) -- ++(0,0,-\cubez) -- ++(\cubex,0,0) -- cycle;

\begin{scope}[shift={(-3,0)}]
\pgfmathsetmacro{\cubex}{0.5}
\pgfmathsetmacro{\cubey}{0.5}
\pgfmathsetmacro{\cubez}{0.5}
\draw[black,fill=yellow!20] (0,0,0) -- ++(-\cubex,0,0) -- ++(0,-\cubey,0) -- ++(\cubex,0,0) -- cycle;
\draw[black,fill=yellow!20] (0,0,0) -- ++(0,0,-\cubez) -- ++(0,-\cubey,0) -- ++(0,0,\cubez) -- cycle;
\draw[black,fill=yellow!20] (0,0,0) -- ++(-\cubex,0,0) -- ++(0,0,-\cubez) -- ++(\cubex,0,0) -- cycle;
\draw (-\cubex,0,0) --  (\cubex,0,0);
\end{scope}

\begin{scope}[shift={(0,1)},rotate=90]
\pgfmathsetmacro{\cubex}{0.5}
\pgfmathsetmacro{\cubey}{0.5}
\pgfmathsetmacro{\cubez}{3}
\draw[black,fill=black!20] (0,0,0) -- ++(-\cubex,0,0) -- ++(0,-\cubey,0) -- ++(\cubex,0,0) -- cycle;
\draw[black,fill=black!20] (0,0,0) -- ++(0,0,-\cubez) -- ++(0,-\cubey,0) -- ++(0,0,\cubez) -- cycle;
\draw[black,fill=black!20] (0,0,0) -- ++(-\cubex,0,0) -- ++(0,0,-\cubez) -- ++(\cubex,0,0) -- cycle;

\pgfmathsetmacro{\cubex}{0.5}
\pgfmathsetmacro{\cubey}{0.5}
\pgfmathsetmacro{\cubez}{0.5}
\draw[black,fill=red!20] (0,0,0) -- ++(-\cubex,0,0) -- ++(0,-\cubey,0) -- ++(\cubex,0,0) -- cycle;
\draw[black,fill=red!20] (0,0,0) -- ++(0,0,-\cubez) -- ++(0,-\cubey,0) -- ++(0,0,\cubez) -- cycle;
\draw[black,fill=red!20] (0,0,0) -- ++(-\cubex,0,0) -- ++(0,0,-\cubez) -- ++(\cubex,0,0) -- cycle;

\begin{scope}[shift={(1,1)}]
\pgfmathsetmacro{\cubex}{0.5}
\pgfmathsetmacro{\cubey}{0.5}
\pgfmathsetmacro{\cubez}{0.5}
\draw[black,fill=green!20] (0,0,0) -- ++(-\cubex,0,0) -- ++(0,-\cubey,0) -- ++(\cubex,0,0) -- cycle;
\draw[black,fill=green!20] (0,0,0) -- ++(0,0,-\cubez) -- ++(0,-\cubey,0) -- ++(0,0,\cubez) -- cycle;
\draw[black,fill=green!20] (0,0,0) -- ++(-\cubex,0,0) -- ++(0,0,-\cubez) -- ++(\cubex,0,0) -- cycle;
\draw (0,0,-\cubez) --  (0,0,3*\cubez);
\end{scope}
\end{scope}

\begin{scope}[shift={(-1.5,2.5)}, rotate=-45]
\pgfmathsetmacro{\cubex}{0.5}
\pgfmathsetmacro{\cubey}{3}
\pgfmathsetmacro{\cubez}{0.5}
\draw[black,fill=black!20] (0,0,0) -- ++(-\cubex,0,0) -- ++(0,-\cubey,0) -- ++(\cubex,0,0) -- cycle;
\draw[black,fill=black!20] (0,0,0) -- ++(0,0,-\cubez) -- ++(0,-\cubey,0) -- ++(0,0,\cubez) -- cycle;
\draw[black,fill=black!20] (0,0,0) -- ++(-\cubex,0,0) -- ++(0,0,-\cubez) -- ++(\cubex,0,0) -- cycle;

\pgfmathsetmacro{\cubex}{0.5}
\pgfmathsetmacro{\cubey}{0.5}
\pgfmathsetmacro{\cubez}{0.5}
\draw[black,fill=green!20] (0,0,0) -- ++(-\cubex,0,0) -- ++(0,-\cubey,0) -- ++(\cubex,0,0) -- cycle;
\draw[black,fill=green!20] (0,0,0) -- ++(0,0,-\cubez) -- ++(0,-\cubey,0) -- ++(0,0,\cubez) -- cycle;
\draw[black,fill=green!20] (0,0,0) -- ++(-\cubex,0,0) -- ++(0,0,-\cubez) -- ++(\cubex,0,0) -- cycle;

\begin{scope}[shift={(0,-3)}]
\pgfmathsetmacro{\cubex}{0.5}
\pgfmathsetmacro{\cubey}{0.5}
\pgfmathsetmacro{\cubez}{0.5}
\draw[black,fill=yellow!20] (0,0,0) -- ++(-\cubex,0,0) -- ++(0,-\cubey,0) -- ++(\cubex,0,0) -- cycle;
\draw[black,fill=yellow!20] (0,0,0) -- ++(0,0,-\cubez) -- ++(0,-\cubey,0) -- ++(0,0,\cubez) -- cycle;
\draw[black,fill=yellow!20] (0,0,0) -- ++(-\cubex,0,0) -- ++(0,0,-\cubez) -- ++(\cubex,0,0) -- cycle;
\draw (0,-\cubey,0) --  (0,\cubey,0);
\end{scope}
\end{scope}

    \node (a) at (0.5,-0.5) {$L_{12}$};
    \node (b) at (0.9,0.6) {$L_{21}$};
    
    \node (c) at (-0.7,2.4) {$L_{23}$};
    \node (d) at (-1.2,2.9) {$L_{32}$};

     \node (e) at (-4.4,0) {$L_{31}$};
    \node (f) at (-3.9,-0.5) {$L_{13}$};

    \node (g) at (-1.6,-0.8) {$L_{1}$};
     \node (h) at (0.1,1.8) {$L_{2}$};
     \node (i) at (-3.4,1.8) {$L_{3}$};
\end{tikzpicture}
\caption{(a) Projection of tribar into  $\mathbb{R}^2$. (b) Decomposition into three overlapping pieces in $\mathbb{R}^3$.} \label{fig:tribar}
\end{figure}
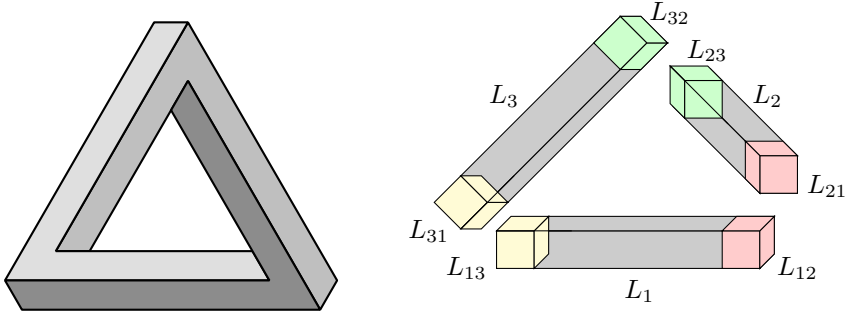

The tribar is more commonly shown in its $2$D projected form as in Figure~\ref{fig:tribar}(a). Let $\Delta$ be the triangular $2$D object in Figure~\ref{fig:tribar}(a), which appears to be the projection of the Penrose tribar, should it exist, onto a plane $H \cong \mathbb{R}^2$. Indeed, there are (infinitely) many $3$D objects that, when projected onto a plane $H \cong \mathbb{R}^2$, gives $\Delta$ as an image. An example is the object in Figure~\ref{fig:tribar3}, as we explain below. 
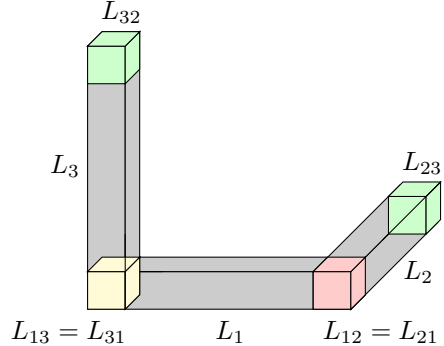
\begin{figure}[h]
\begin{center}
\begin{tikzpicture}
\pgfmathsetmacro{\cubex}{3}
\pgfmathsetmacro{\cubey}{0.5}
\pgfmathsetmacro{\cubez}{0.5}
\draw[black,fill=black!20] (0,0,0) -- ++(-\cubex,0,0) -- ++(0,-\cubey,0) -- ++(\cubex,0,0) -- cycle;
\draw[black,fill=black!20] (0,0,0) -- ++(0,0,-\cubez) -- ++(0,-\cubey,0) -- ++(0,0,\cubez) -- cycle;
\draw[black,fill=black!20] (0,0,0) -- ++(-\cubex,0,0) -- ++(0,0,-\cubez) -- ++(\cubex,0,0) -- cycle;

\pgfmathsetmacro{\cubex}{0.5}
\pgfmathsetmacro{\cubey}{0.5}
\pgfmathsetmacro{\cubez}{0.5}
\draw[black,fill=red!20] (0,0,0) -- ++(-\cubex,0,0) -- ++(0,-\cubey,0) -- ++(\cubex,0,0) -- cycle;
\draw[black,fill=red!20] (0,0,0) -- ++(0,0,-\cubez) -- ++(0,-\cubey,0) -- ++(0,0,\cubez) -- cycle;
\draw[black,fill=red!20] (0,0,0) -- ++(-\cubex,0,0) -- ++(0,0,-\cubez) -- ++(\cubex,0,0) -- cycle;

\begin{scope}[shift={(-3,0)}]
\pgfmathsetmacro{\cubex}{0.5}
\pgfmathsetmacro{\cubey}{0.5}
\pgfmathsetmacro{\cubez}{0.5}
\draw[black,fill=yellow!20] (0,0,0) -- ++(-\cubex,0,0) -- ++(0,-\cubey,0) -- ++(\cubex,0,0) -- cycle;
\draw[black,fill=yellow!20] (0,0,0) -- ++(0,0,-\cubez) -- ++(0,-\cubey,0) -- ++(0,0,\cubez) -- cycle;
\draw[black,fill=yellow!20] (0,0,0) -- ++(-\cubex,0,0) -- ++(0,0,-\cubez) -- ++(\cubex,0,0) -- cycle;
\draw (-\cubex,0,0) --  (\cubex,0,0);
\end{scope}

\begin{scope}[shift={(0,0)},rotate=0]
\pgfmathsetmacro{\cubex}{0.5}
\pgfmathsetmacro{\cubey}{0.5}
\pgfmathsetmacro{\cubez}{3}
\draw[black,fill=black!20] (0,0,0) -- ++(-\cubex,0,0) -- ++(0,-\cubey,0) -- ++(\cubex,0,0) -- cycle;
\draw[black,fill=black!20] (0,0,0) -- ++(0,0,-\cubez) -- ++(0,-\cubey,0) -- ++(0,0,\cubez) -- cycle;
\draw[black,fill=black!20] (0,0,0) -- ++(-\cubex,0,0) -- ++(0,0,-\cubez) -- ++(\cubex,0,0) -- cycle;

\pgfmathsetmacro{\cubex}{0.5}
\pgfmathsetmacro{\cubey}{0.5}
\pgfmathsetmacro{\cubez}{0.5}
\draw[black,fill=red!20] (0,0,0) -- ++(-\cubex,0,0) -- ++(0,-\cubey,0) -- ++(\cubex,0,0) -- cycle;
\draw[black,fill=red!20] (0,0,0) -- ++(0,0,-\cubez) -- ++(0,-\cubey,0) -- ++(0,0,\cubez) -- cycle;
\draw[black,fill=red!20] (0,0,0) -- ++(-\cubex,0,0) -- ++(0,0,-\cubez) -- ++(\cubex,0,0) -- cycle;

\begin{scope}[shift={(1,1)}]
\pgfmathsetmacro{\cubex}{0.5}
\pgfmathsetmacro{\cubey}{0.5}
\pgfmathsetmacro{\cubez}{0.5}
\draw[black,fill=green!20] (0,0,0) -- ++(-\cubex,0,0) -- ++(0,-\cubey,0) -- ++(\cubex,0,0) -- cycle;
\draw[black,fill=green!20] (0,0,0) -- ++(0,0,-\cubez) -- ++(0,-\cubey,0) -- ++(0,0,\cubez) -- cycle;
\draw[black,fill=green!20] (0,0,0) -- ++(-\cubex,0,0) -- ++(0,0,-\cubez) -- ++(\cubex,0,0) -- cycle;
\draw (0,0,-\cubez) --  (0,0,3*\cubez);
\end{scope}
\end{scope}

\begin{scope}[shift={(-3,3)}, rotate=0]
\pgfmathsetmacro{\cubex}{0.5}
\pgfmathsetmacro{\cubey}{3}
\pgfmathsetmacro{\cubez}{0.5}
\draw[black,fill=black!20] (0,0,0) -- ++(-\cubex,0,0) -- ++(0,-\cubey,0) -- ++(\cubex,0,0) -- cycle;
\draw[black,fill=black!20] (0,0,0) -- ++(0,0,-\cubez) -- ++(0,-\cubey,0) -- ++(0,0,\cubez) -- cycle;
\draw[black,fill=black!20] (0,0,0) -- ++(-\cubex,0,0) -- ++(0,0,-\cubez) -- ++(\cubex,0,0) -- cycle;

\pgfmathsetmacro{\cubex}{0.5}
\pgfmathsetmacro{\cubey}{0.5}
\pgfmathsetmacro{\cubez}{0.5}
\draw[black,fill=green!20] (0,0,0) -- ++(-\cubex,0,0) -- ++(0,-\cubey,0) -- ++(\cubex,0,0) -- cycle;
\draw[black,fill=green!20] (0,0,0) -- ++(0,0,-\cubez) -- ++(0,-\cubey,0) -- ++(0,0,\cubez) -- cycle;
\draw[black,fill=green!20] (0,0,0) -- ++(-\cubex,0,0) -- ++(0,0,-\cubez) -- ++(\cubex,0,0) -- cycle;

\begin{scope}[shift={(0,-3)}]
\pgfmathsetmacro{\cubex}{0.5}
\pgfmathsetmacro{\cubey}{0.5}
\pgfmathsetmacro{\cubez}{0.5}
\draw[black,fill=yellow!20] (0,0,0) -- ++(-\cubex,0,0) -- ++(0,-\cubey,0) -- ++(\cubex,0,0) -- cycle;
\draw[black,fill=yellow!20] (0,0,0) -- ++(0,0,-\cubez) -- ++(0,-\cubey,0) -- ++(0,0,\cubez) -- cycle;
\draw[black,fill=yellow!20] (0,0,0) -- ++(-\cubex,0,0) -- ++(0,0,-\cubez) -- ++(\cubex,0,0) -- cycle;
\draw (0,-\cubey,0) --  (0,\cubey,0);
\end{scope}
\end{scope}

    \node (a) at (0.4,-0.8) {$L_{12} = L_{21}$};
    
    \node (c) at (0.95,1.45) {$L_{23}$};
    \node (d) at (-3.05,3.45) {$L_{32}$};

     \node (e) at (-3.75,-0.8) {$L_{13}=L_{31}$};

    \node (g) at (-1.6,-0.8) {$L_{1}$};
     \node (h) at (0.9,0) {$L_{2}$};
     \node (i) at (-3.8,1.4) {$L_{3}$};
\end{tikzpicture}
\end{center}
\caption{A $3$D object whose projection onto $\mathbb{R}^2$ is $\Delta$.}\label{fig:tribar3}
\end{figure}

Note that the object in Figure~\ref{fig:tribar3} is an abstraction of the sculpture in Figure~\ref{fig:perth}, which depicts how it projects to give $\Delta$ when viewed from an appropriate angle. The plane $H$ in this case is either the viewer's retina or the camera's photographic film.

\begin{figure}[h]
\centering
\includegraphics[width=\textwidth]{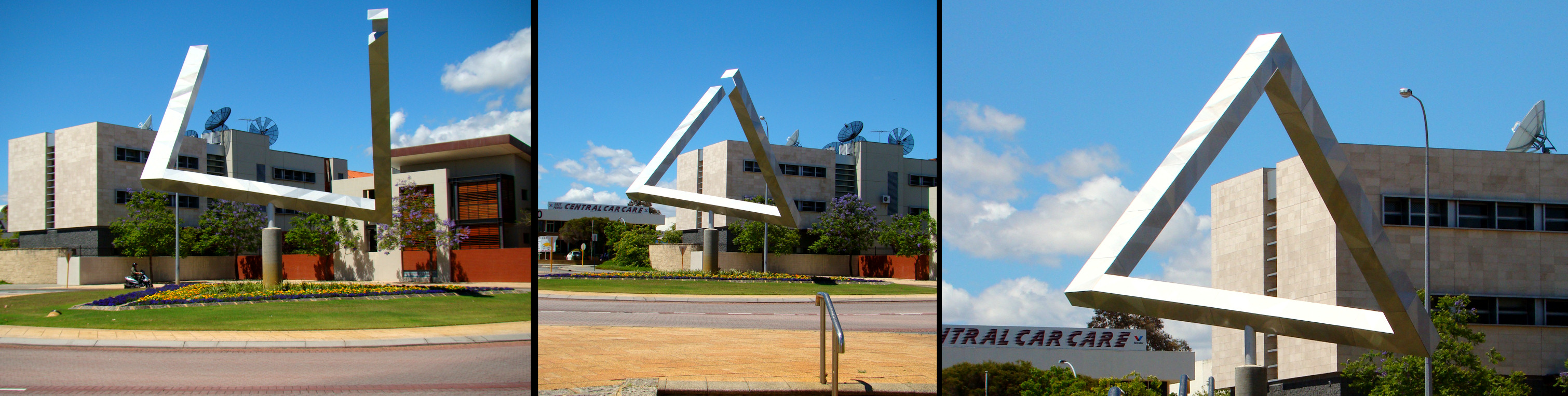}
\caption{The \emph{Impossible Triangle} sculpture by Brian MacKay and Ahmad Abas, located in the Claisebrook Roundabout, Perth, Australia. Photograph by Bj\o rn Christian T\o rrissen shared under a Creative Commons license.}
\label{fig:perth}
\end{figure}

Let $H \subseteq \mathbb{R}^3$ be a hyperplane, which partitions $\mathbb{R}^3$ into two half-spaces. Let $O \in \mathbb{R}^3$ be an arbitrary point in one half-space and the three bars  $L_1, L_2, L_3$ be in the other.  The reader should think of $O$ as the position of the viewer and the viewing direction as a normal to $H$. Now we are going to arrange $L_1, L_2, L_3$ in such a way that their projections onto $H$ give us $\Delta$. This is clearly possible; for example, the $3$D object in Figure~\ref{fig:tribar3}, upon an appropriate rotation dependent on $H$ and $O$, would give $\Delta$ as a projection.


Define $d_{ij} \in \mathbb{R}_+$ to be the distance from $O$ to the center of $L_{ij}$ and $d_{ii}=1$, $i,j=1,2,3$. Let $g = (g_{ij})_{i,j=1}^3 $ be the $3 \times 3$ matrix of cross ratios
\[
g_{ij} = \frac{d_{ij}}{d_{ji}}, \qquad i,j =1,2,3.
\]
Then $g$ is a matrix with $g_{ij}^{-1} = g_{ji}$ and $g_{ii} = 1$ for all $i,j =1,2,3$.

The matrix $g$ is a function of the positions of the bars $L_1, L_2, L_3$, or, to be precise, a function of the centroids of these rigid bodies. These bars have a certain degree of freedom: We may \emph{move} each of them independently along the viewing direction and this would keep their projections in $\mathbb{R}^2$ invariant, always forming $\Delta$. This movement is a similarity transform that preserves the direction of the bar, with no rotation. Moving $L_i$ in the viewing direction results in a rescaling of the distance $d_{ij}$ by a factor  $g_i \in \mathbb{R}_+$ for all $j\ne i$, i.e., if $d_{ij}'$ denotes the new distance upon moving  $L_i$'s along viewing directions, then $d_{ij}' = d_{ij}/g_i$, for all $i\ne j$. Let  $g' = (g'_{ij})_{i,j=1}^3$ be the new matrix of cross ratios upon moving $L_i$'s along viewing direction. Then we have
\begin{equation}\label{eq:cross}
g'_{ij} = \frac{d_{ij}'}{d_{ji}'} = \frac{d_{ij}/g_i}{d_{ji}/g_j} = g _{ij} \frac{g_j} {g_i}, \qquad i,j =1,2,3.
\end{equation}

Suppose that we could eventually move  $L_1, L_2, L_3$ to form the tribar in $\mathbb{R}^3$. Then, in this final position, the centers of $L_{ij}$ and $L_{ji}$ coincide and so $d_{ij}' = d_{ji}'$ for all $ i \ne j$, and thus $g_{ij}'=1$ for all $i,j =1,2,3$. In other words, the matrix $g$ must be a \textit{coboundary}, i.e.,
\begin{equation}\label{eq:cob}
g_{ij} =\frac{g_i}{g_j},
\end{equation}
for some $g_i, g_j \in \mathbb{R}_+$, $i,j =1,2,3$.

In summary, what we have shown is that if $L_1, L_2, L_3$ could be moved into place to form a tribar, then for  $L_1, L_2, L_3$ in \textit{any} positions that form $\Delta$ upon projection onto $\mathbb{R}^2$,  the corresponding matrix $g$ must be a coboundary, i.e., it satisfies \eqref{eq:cob}, or equivalently, $g$ is the identity element in the cohomology group $H^1(\mathbb{R}^2,\mathbb{R}_+)$. 
With this observation, we will next derive a contradiction showing that the tribar does not exist. Let $L_1, L_2, L_3$ be arranged as in Figure~\ref{fig:tribar3} and recall that their projections onto $\mathbb{R}^2$ give $\Delta$. In this case, the matrix $g$ is 
\[
g = 
\begin{bmatrix}
1 & 1 & 1 \\
1 & 1 & g_{23}\\
1 & g_{32} & 1
\end{bmatrix}.
\]
If the tribar exists, then $g$ is a coboundary, i.e., \eqref{eq:cob} has a solution for some $g_i, g_j \in \mathbb{R}_+$, $i,j =1,2,3$, and so 
\[
g_1 = g_2 = g_3, 
\]
implying $g_{23} = 1$. However, as is evident from Figure~\ref{fig:tribar3}, $L_{23}$ does not even intersect $L_{32}$ and so $g_{23} \ne 1$, a contradiction.

Although the tribar does not exist as a $3$D object, i.e., it cannot be embedded in $\mathbb{R}^3$, it clearly exists as an abstract geometrical object (a cubical complex) defined by the gluing procedure described earlier --- we will call this the \textit{intrinsic tribar} to distinguish it from the nonexistent $3$D object. In fact, the intrinsic tribar can be embedded in a three-dimensional manifold $\mathbb{R}^3/\mathbb{Z}$, a quotient space of $\mathbb{R}^3$ under a certain action of the discrete group $\mathbb{Z}$ related to Figure~\ref{fig:tribar3} (see \cite{Francis} for details).

We emphasize that a tribar is a geometrical object, not a topological one. It may be tempting to draw a parallel between the non-embeddability of the intrinsic tribar in $\mathbb{R}^3$ with the non-embeddability of the M\"obius strip in $\mathbb{R}^2$ or the Klein bottle in $\mathbb{R}^3$. But these are different phenomena. As a topological object, a M\"obius strip is only defined up to homotopy, i.e., we may freely deform a M\"obius strip continuously. However the definition of the tribar does not afford this flexibility, i.e., a tribar is not homotopy invariant. For instance, we are not allowed to twist or bend the bars. In fact, had we allowed such continuous deformation, the intrinsic tribar is homotopy equivalent to a torus and therefore trivially embeddable in $\mathbb{R}^3$. This is much like our study of cryo-EM, where the goal is to reconstruct the $3$D structure of a molecule precisely, and not just up to homotopy.

\begin{figure}[h]
\begin{center}
    \begin{tikzpicture}[scale=2.5, line join=bevel]
	
      \pgfmathsetmacro{\a}{0.3}
      \pgfmathsetmacro{\b}{1.5}

      \tikzset{%
        apply style/.code={\tikzset{#1}},
        brick_edges/.style={thick,draw=black},
        face_colourA/.style={fill=gray!50},
        face_colourB/.style={fill=gray!25},
        face_colourC/.style={fill=gray!90},
      }

      \foreach \theta/\v/\facestyleone/\facestyletwo in {%
        0/0/{brick_edges,face_colourA}/{brick_edges,face_colourC},
        180/-\a/{brick_edges,face_colourB}/{brick_edges,face_colourC}
      }{
      \begin{scope}[rotate=\theta,shift={(\v,0)}]
        \draw[apply style/.expand once=\facestyleone]  		
          ({-.5*\b},{1.5*\a}) --
          ++(\b,0)            --
          ++(-\a,-\a)         --
          ++({-\b+2*\a},0)    --
          ++(0,-{2*\a})       --
          ++(\b,0)            --
          ++(-\a,-\a)         --
          ++(-\b,0)           --
          cycle;
        \draw[apply style/.expand once=\facestyletwo] 
          ({.5*\b},{1.5*\a})  --
          ++(0,{-2*\a})       --
          ++(-\a,0)           --
          ++(0,\a)            --
          cycle;
        \end{scope}
      }
    \end{tikzpicture}
\qquad
\begin{tikzpicture}
\pgfmathsetmacro{\cubex}{3}
\pgfmathsetmacro{\cubey}{0.5}
\pgfmathsetmacro{\cubez}{0.5}
\draw[black,fill=black!20] (0,0,0) -- ++(-\cubex,0,0) -- ++(0,-\cubey,0) -- ++(\cubex,0,0) -- cycle;
\draw[black,fill=black!20] (0,0,0) -- ++(0,0,-\cubez) -- ++(0,-\cubey,0) -- ++(0,0,\cubez) -- cycle;
\draw[black,fill=black!20] (0,0,0) -- ++(-\cubex,0,0) -- ++(0,0,-\cubez) -- ++(\cubex,0,0) -- cycle;
\pgfmathsetmacro{\cubex}{0.5}
\pgfmathsetmacro{\cubey}{0.5}
\pgfmathsetmacro{\cubez}{0.5}
\draw[black,fill=red!20] (0,0,0) -- ++(-\cubex,0,0) -- ++(0,-\cubey,0) -- ++(\cubex,0,0) -- cycle;
\draw[black,fill=red!20] (0,0,0) -- ++(0,0,-\cubez) -- ++(0,-\cubey,0) -- ++(0,0,\cubez) -- cycle;
\draw[black,fill=red!20] (0,0,0) -- ++(-\cubex,0,0) -- ++(0,0,-\cubez) -- ++(\cubex,0,0) -- cycle;
\begin{scope}[shift={(-3,0)}]
\pgfmathsetmacro{\cubex}{0.5}
\pgfmathsetmacro{\cubey}{0.5}
\pgfmathsetmacro{\cubez}{0.5}
\draw[black,fill=yellow!20] (0,0,0) -- ++(-\cubex,0,0) -- ++(0,-\cubey,0) -- ++(\cubex,0,0) -- cycle;
\draw[black,fill=yellow!20] (0,0,0) -- ++(0,0,-\cubez) -- ++(0,-\cubey,0) -- ++(0,0,\cubez) -- cycle;
\draw[black,fill=yellow!20] (0,0,0) -- ++(-\cubex,0,0) -- ++(0,0,-\cubez) -- ++(\cubex,0,0) -- cycle;
\draw (-\cubex,0,0) --  (\cubex,0,0);
\end{scope}

\begin{scope}[shift={(1,3)}]
\pgfmathsetmacro{\cubex}{0.5}
\pgfmathsetmacro{\cubey}{3.5}
\pgfmathsetmacro{\cubez}{0.5}
\draw[black,fill=black!20] (0,0,0) -- ++(-\cubex,0,0) -- ++(0,-\cubey,0) -- ++(\cubex,0,0) -- cycle;
\draw[black,fill=black!20] (0,0,0) -- ++(0,0,-\cubez) -- ++(0,-\cubey,0) -- ++(0,0,\cubez) -- cycle;
\draw[black,fill=black!20] (0,0,0) -- ++(-\cubex,0,0) -- ++(0,0,-\cubez) -- ++(\cubex,0,0) -- cycle;

\pgfmathsetmacro{\cubex}{0.5}
\pgfmathsetmacro{\cubey}{0.5}
\pgfmathsetmacro{\cubez}{0.5}
\draw[black,fill=green!20] (0,0,0) -- ++(-\cubex,0,0) -- ++(0,-\cubey,0) -- ++(\cubex,0,0) -- cycle;
\draw[black,fill=green!20] (0,0,0) -- ++(0,0,-\cubez) -- ++(0,-\cubey,0) -- ++(0,0,\cubez) -- cycle;
\draw[black,fill=green!20] (0,0,0) -- ++(-\cubex,0,0) -- ++(0,0,-\cubez) -- ++(\cubex,0,0) -- cycle;

\begin{scope}[shift={(0,-3)}]
\pgfmathsetmacro{\cubex}{0.5}
\pgfmathsetmacro{\cubey}{0.5}
\pgfmathsetmacro{\cubez}{0.5}
\draw[black,fill=red!20] (0,0,0) -- ++(-\cubex,0,0) -- ++(0,-\cubey,0) -- ++(\cubex,0,0) -- cycle;
\draw[black,fill=red!20] (0,0,0) -- ++(0,0,-\cubez) -- ++(0,-\cubey,0) -- ++(0,0,\cubez) -- cycle;
\draw[black,fill=red!20] (0,0,0) -- ++(-\cubex,0,0) -- ++(0,0,-\cubez) -- ++(\cubex,0,0) -- cycle;
\draw (0,-\cubey,0) --  (0,\cubex,0);
\end{scope}
\end{scope}

\begin{scope}[shift={(-4,3)}]
\pgfmathsetmacro{\cubex}{0.5}
\pgfmathsetmacro{\cubey}{3.5}
\pgfmathsetmacro{\cubez}{0.5}
\draw[black,fill=black!20] (0,0,0) -- ++(-\cubex,0,0) -- ++(0,-\cubey,0) -- ++(\cubex,0,0) -- cycle;
\draw[black,fill=black!20] (0,0,0) -- ++(0,0,-\cubez) -- ++(0,-\cubey,0) -- ++(0,0,\cubez) -- cycle;
\draw[black,fill=black!20] (0,0,0) -- ++(-\cubex,0,0) -- ++(0,0,-\cubez) -- ++(\cubex,0,0) -- cycle;

\pgfmathsetmacro{\cubex}{0.5}
\pgfmathsetmacro{\cubey}{0.5}
\pgfmathsetmacro{\cubez}{0.5}
\draw[black,fill=purple!20] (0,0,0) -- ++(-\cubex,0,0) -- ++(0,-\cubey,0) -- ++(\cubex,0,0) -- cycle;
\draw[black,fill=purple!20] (0,0,0) -- ++(0,0,-\cubez) -- ++(0,-\cubey,0) -- ++(0,0,\cubez) -- cycle;
\draw[black,fill=purple!20] (0,0,0) -- ++(-\cubex,0,0) -- ++(0,0,-\cubez) -- ++(\cubex,0,0) -- cycle;

\begin{scope}[shift={(0,-3)}]
\pgfmathsetmacro{\cubex}{0.5}
\pgfmathsetmacro{\cubey}{0.5}
\pgfmathsetmacro{\cubez}{0.5}
\draw[black,fill=yellow!20] (0,0,0) -- ++(-\cubex,0,0) -- ++(0,-\cubey,0) -- ++(\cubex,0,0) -- cycle;
\draw[black,fill=yellow!20] (0,0,0) -- ++(0,0,-\cubez) -- ++(0,-\cubey,0) -- ++(0,0,\cubez) -- cycle;
\draw[black,fill=yellow!20] (0,0,0) -- ++(-\cubex,0,0) -- ++(0,0,-\cubez) -- ++(\cubex,0,0) -- cycle;
\draw (0,-\cubey,0) --  (0,\cubex,0);
\end{scope}
\end{scope}
\begin{scope}[shift={(0,3)}]
\pgfmathsetmacro{\cubex}{3}
\pgfmathsetmacro{\cubey}{0.5}
\pgfmathsetmacro{\cubez}{0.5}
\draw[black,fill=black!20] (0,0,0) -- ++(-\cubex,0,0) -- ++(0,-\cubey,0) -- ++(\cubex,0,0) -- cycle;
\draw[black,fill=black!20] (0,0,0) -- ++(0,0,-\cubez) -- ++(0,-\cubey,0) -- ++(0,0,\cubez) -- cycle;
\draw[black,fill=black!20] (0,0,0) -- ++(-\cubex,0,0) -- ++(0,0,-\cubez) -- ++(\cubex,0,0) -- cycle;
\pgfmathsetmacro{\cubex}{0.5}
\pgfmathsetmacro{\cubey}{0.5}
\pgfmathsetmacro{\cubez}{0.5}
\draw[black,fill=green!20] (0,0,0) -- ++(-\cubex,0,0) -- ++(0,-\cubey,0) -- ++(\cubex,0,0) -- cycle;
\draw[black,fill=green!20] (0,0,0) -- ++(0,0,-\cubez) -- ++(0,-\cubey,0) -- ++(0,0,\cubez) -- cycle;
\draw[black,fill=green!20] (0,0,0) -- ++(-\cubex,0,0) -- ++(0,0,-\cubez) -- ++(\cubex,0,0) -- cycle;
\begin{scope}[shift={(-3,0)}]
\pgfmathsetmacro{\cubex}{0.5}
\pgfmathsetmacro{\cubey}{0.5}
\pgfmathsetmacro{\cubez}{0.5}
\draw[black,fill=purple!20] (0,0,0) -- ++(-\cubex,0,0) -- ++(0,-\cubey,0) -- ++(\cubex,0,0) -- cycle;
\draw[black,fill=purple!20] (0,0,0) -- ++(0,0,-\cubez) -- ++(0,-\cubey,0) -- ++(0,0,\cubez) -- cycle;
\draw[black,fill=purple!20] (0,0,0) -- ++(-\cubex,0,0) -- ++(0,0,-\cubez) -- ++(\cubex,0,0) -- cycle;
\draw (-\cubex,0,0) --  (\cubex,0,0);
\end{scope}
\end{scope}

    \node (a) at (0,-0.8) {$L_{12}$};
    \node (b) at (1.6,-0.4) {$L_{21}$};
    
      \node (c) at (1.6,3) {$L_{23}$};
     \node (c) at (0,3.5) {$L_{32}$};
     
    \node (d) at (-3,3.5) {$L_{34}$};
     \node (d) at (-4.9,3) {$L_{43}$};
     
     \node (e) at (-4.9,-0.4) {$L_{41}$};
    \node (f) at (-3.2,-0.8) {$L_{14}$};
    
    \node (g) at (-1.6,-0.8) {$L_{1}$};
     \node (h) at (1.5,1.35) {$L_{2}$};
     \node (g) at (-1.6,3.5) {$L_{3}$};
     \node (i) at (-4.8,1.35) {$L_{4}$};
\end{tikzpicture}
\caption{(a) Projection of  Escher brick into $\mathbb{R}^2$. (b) Decomposition into overlapping pieces in $\mathbb{R}^3$.}  \label{fig:brick}
\end{center}
\end{figure}
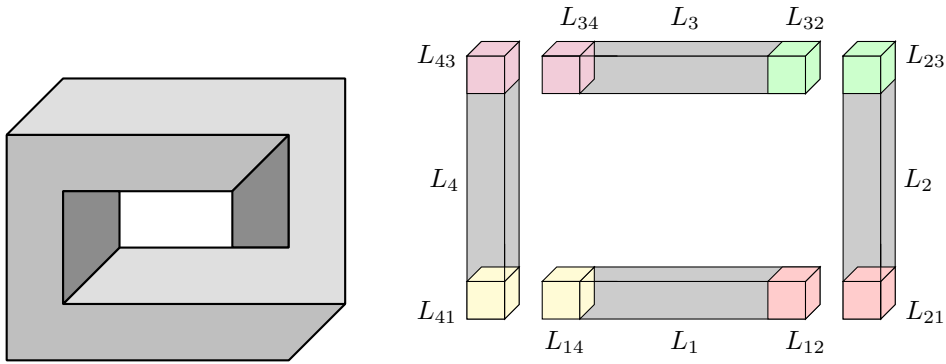

The discussions above also apply to other impossible objects in $\mathbb{R}^3$. For example, the \textit{Escher brick}, defined as the (nonexistent) $3$D object obtained by gluing four bars $L_1, L_2, L_3, L_4$ as in Figure~\ref{fig:brick}. If the Escher brick exists in $\mathbb{R}^3$, then whenever  $L_1, L_2, L_3, L_4$ projects onto $\mathbb{R}^2$ to form Figure~\ref{fig:brick}(a), the matrix $g \in  \mathbb{R}^{4 \times 4}$ is necessarily a coboundary, i.e., satisfies $g_{ij} = g_i/g_j$ for some $g_i \in \mathbb{R}_+$, $i,j=1,2,3,4$. We may construct an analogue of Figure~\ref{fig:tribar3} whereby we glue three of the four ends in Figure~\ref{fig:brick}(b). This $3$D object projects onto $\mathbb{R}^2$ to form  Figure~\ref{fig:brick}(a) but its corresponding  matrix $g \in  \mathbb{R}^{4 \times 4}$ is not a coboundary. Hence the Escher brick does not exist in $\mathbb{R}^3$.

\section{Singular Cohomology and {\v C}ech Cohomology}\label{sec:three}

This article is primarily intended for an applied and computational mathematics readership. For readers unfamiliar with algebraic topology, this section provides in one place all the required definitions and background material, kept to a bare minimum of just what we need for this article.

We will define two types of cohomology groups associated to a topological space $X$ and a topological group $G$ that will be useful for our study of the cryo-EM problem: $H^n(X,G)$, the singular cohomology group with coefficients in $G$; and $\check{H}^n(X,G)$, the {\v C}ech cohomology group with coefficients in  $G$. For a given $X$, these cohomology groups are in general different; but they would always be isomorphic for the space $X$ that we construct from a given collection of cryo-EM images (see Section~\ref{sec:main}). The reason we need both of them is that they are good for different purposes: the cohomology of cryo-EM is most naturally formulated in terms of {\v C}ech cohomology; but singular cohomology is more readily computable and facilitates our explicit calculations.

Our descriptions in the next few subsections are highly condensed, but in principle complete and self-contained. While this material is standard, our goal here is to make them accessible to practitioners by limiting the prerequisite to a few rudimentary definitions in point set topology and group theory. We provide pointers to standard sources at the beginning of each subsection.

We use $X \simeq Y$ to denote isomorphism if $X,Y$ are groups, homotopy equivalence if $X,Y$ are topological spaces, and bundle isomorphism if $X,Y$ are bundles. We use $X \cong Y$ to denote homeomorphism of topological spaces.

\subsection{Singular cohomology}\label{sec:sing}

Standard references for this section are \cite{Hatcher, May, Spanier}.

The \textit{standard $n$-simplex} for $n=0,1,2,3$, is the set
\[
\Delta_n \coloneqq \left\{(t_0,\dots, t_n)\in \mathbb{R}^{n+1}: \sum\nolimits_{i=0}^n t_i =1,\;  t_i\ge 0 \right\}.
\]
$\Delta_n$ is the convex hull of its  $n+1$ \textit{vertices},
\[
e_{0} = (0,0,\dots, 0),\;
e_{1} = (1,0,\dots,0),
\dots,
e_{n} = (0,0,\dots, 1).   
\]
The standard $0$-simplex is a point, the standard $1$-simplex is a line, the standard $2$-simplex is a triangle, and the standard $3$-simplex is a tetrahedron. 

For $n=0,1,2$, the convex hull of any $n$ vertices $e_{i_1},\dots, e_{i_n}$ of $\Delta_n$, where $0\le i_1 < \dots < i_n\le n$, is called  a \textit{face} of $\Delta_n$ and denoted by $[i_1,\dots,i_n]$.

Let $X$ be a topological space and $n=0,1,2,3$. A continuous map $\sigma: \Delta_n \to X$ is called a \textit{singular simplicial simplex} on $X$. We denote by $C_n(X)$ the free abelian group generated by all singular simplicial simplices on $X$. The \textit{boundary maps} are homomorphisms of abelian groups
\[
\partial_1: C_1(X) \to C_0(X), \qquad
\partial_2: C_2(X) \to C_1(X), \qquad
\partial_3: C_3(X) \to C_2(X),
\]
defined respectively by the linear extensions of
\begin{align*}
\partial_1 (\sigma) &= \sigma|_{[1]} - \sigma|_{[0]},\\
\partial_2 (\sigma) &= \sigma|_{[1,2]} - \sigma|_{[0,2]} + \sigma|_{[0,1]},\\
\partial_3 (\sigma) &= \sigma|_{[1,2,3]} - \sigma|_{[0,2,3]} + \sigma|_{[0,1,3]} - \sigma|_{[0,1,2]}.
\end{align*}
Here  $\sigma|_{[i]}$ denotes the restriction of $\sigma$ to the face $[i]$ of $\Delta_1$, $\sigma|_{[i,j]}$ denotes the restriction of $\sigma$ to the face $[i,j]$ of $\Delta_2$, and $\sigma|_{[i,j,k]}$ denotes the restriction of $\sigma$ to the face $[i,j,k]$ of $\Delta_3$. We set $\partial_0: C_0(X) \to \{0\}$ to be the zero map.

The sequence of homomorphisms of abelian groups
\begin{equation}\label{chain complex}
C_{3}(X) \xrightarrow{\partial_{3}}  C_{2}(X) \xrightarrow{\partial_2} C_{1}(X)  \xrightarrow{\partial_1} C_{0}(X) \xrightarrow{\partial_0} 0
\end{equation}
forms a \textit{chain complex}, i.e., it has the property that
\begin{equation}\label{eq:bb}
\partial_0 \circ \partial_1 = 0,\qquad \partial_1 \circ \partial_2 = 0, \qquad \partial_2 \circ \partial_3 = 0,
\end{equation}
which are easy to verify.
For $n=0,1,2$, let $Z_n(X) \coloneqq \operatorname{Ker}\partial_n \subseteq C_n(X)$ be the subgroup of \textit{$n$-cycles} and $B_n(X) \coloneqq \operatorname{Im}\partial_{n+1} \subseteq C_n(X)$ be the subgroup of \textit{$n$-boundaries}. It follows from \eqref{eq:bb} that $B_n(X) \subseteq C_n(X)$. The quotient group
\[
H_n(X) \coloneqq Z_n(X)/B_n(X)
\]
is called the $n$th \textit{singular homology group} of $X$, $n = 0,1,2$.

For $n =0,1,2,3$, define $C^n(X) = \operatorname{Hom}_{\mathbb{Z} }(C_n(X),\mathbb{Z})$, the set of all group homomorphisms from $C_n(X)$ to $\mathbb{Z}$. $C^n(X)$ is clearly an abelian group itself under addition of homomorphisms. The map \emph{induced} by the boundary map $\partial_n: C_n(X) \to C_{n-1}(X)$ is defined as
\[
\partial_n^{\ast}: C^{n-1}(X) \to C^n(X), \qquad \partial_n^\ast (f)(\sigma) = f(\partial_n(\sigma)), 
\]
for any $f\in C^{n-1}(X)$ and $ \sigma\in C_n(X)$.
The sequence of homomorphisms of abelian groups
\begin{equation}\label{cochain complex}
0 \xrightarrow{\partial_0^*}  C^{0}(X) \xrightarrow{\partial_1^*} C^1(X) \xrightarrow{\partial_{2}^*} C^{2}(X) \xrightarrow{\partial_{3}^*} C^{3}(X)
\end{equation}
forms a \textit{cochain complex}, i.e., it has the property that
\begin{equation}\label{eq:cc}
\partial_1^* \circ \partial_0^* = 0, \qquad \partial_2^* \circ \partial_1^* = 0, \qquad \partial_3^* \circ \partial_2^* = 0,
\end{equation}
which follows from \eqref{eq:bb}.
For $n=0,1,2$, let $Z^n(X) \coloneqq \operatorname{Ker}\partial_{n+1}^* \subseteq C^n(X)$ be the subgroup of \textit{$n$-cocycles} and $B^n(X) \coloneqq\operatorname{Im}\partial_n^* \subseteq C^n(X)$ be the subgroup of \textit{$n$-coboundaries}. The quotient group
\[
H^n(X) \coloneqq Z^n(X)/B^n(X)
\]
is called the $n$th \textit{singular cohomology group} of $X$, $n=0,1,2$. More generally, let $G$ be a group then one can define the $n$th \textit{singular cohomology group} $H^n(X,G)$ \textit{with coefficient $G$} of $X$ to be the cohomology groups $Z^n(X,G)/B^n(X,G)$ of the cochain complex
\[
0 \xrightarrow{\partial_0^*}  C^{0}(X,G) \xrightarrow{\partial_1^*} C^1(X,G) \xrightarrow{\partial_{2}^*} C^{2}(X,G) \xrightarrow{\partial_{3}^*} C^{3}(X,G)
\]
where $C^n(X,G) = \operatorname{Hom}_{\mathbb{Z}} (C_n(X), G)$, $\partial_n^*$ is the map induced by $\partial_n : C_n(X) \to C_{n-1}(X)$, $n=0,1,2$ and 
\begin{align*}
Z^n(X,G) & \coloneqq \operatorname{Ker}\partial_{n+1}^* \subseteq C^n(X,G), \\
B^n(X,G) & \coloneqq\operatorname{Im}\partial_n^* \subseteq C^n(X,G).
\end{align*}
Note that when $G = \mathbb{Z}$, $C^n(X,\mathbb{Z}) = C^n(X)$, $Z^n(X,\mathbb{Z}) = Z^n(X)$,  $B^n(X,\mathbb{Z}) = B^n(X)$, $H^n(X,\mathbb{Z})=H^n(X)$.

For the purpose of this paper, $X$ would take the form of a \textit{finite simplicial complex}, a collection $K$ of finitely many simplices such that
\begin{enumerate}[\upshape (i)]
\item every face of a simplex in $K$ is also contained in $K$;
\item the intersection of two simplices $\Delta_1,\Delta_2$ in $K$ is a face of both $\Delta_1$ and $\Delta_2$.
\end{enumerate}
We denote the union of simplices in $K$ by $|K|$. We also say that a topological space $X$ is a \textit{finite simplicial complex} if $X$ can be realized as $|K|$ for some finite simplicial complex $K$. For example, spheres $\mathbb{S}^n$ and tori $\mathbb{S}^1 \times \dots \times \mathbb{S}^1$ are finite simplicial complexes in this more general sense.

For the purpose of this paper, readers only need to know that
\[
H_0(\mathbb{S}^2) \simeq H_2(\mathbb{S}^2) \simeq \mathbb{Z},\qquad H_1(\mathbb{S}^2) = 0,\qquad
H^0(\mathbb{S}^2) \simeq H^2(\mathbb{S}^2) \simeq \mathbb{Z},\qquad H^1(\mathbb{S}^2) = 0,
\]
and that if $X$ is a simplicial complex of dimension $p$, then $H_n (X) = 0$ for all $n> p$.

A topological space $X$ is \textit{contractible} if there is a point $x_0 \in X$ and a continuous map $H : X\times [0,1] \to X$ such that 
\[
H(x,0) = x_0 \qquad \text{and}\qquad H(x,1) = x.
\]
Roughly speaking, this means that $X$ can be continuously shrunk to a point $x_0$.
For example, an open/closed/half-open-half-closed line segment is contractible, as is an open/closed disk or a disk with an arc on the boundary. The following is the only fact about contractible spaces that we need for this article.
\begin{proposition}\label{prop:contractible}
If $X$ is contractible and $G$ is an abelian group, then $H^n (X,G) = 0$ for all $n>0$ and $H^0 (X,G) = G$.
\end{proposition}

\subsection{Principal bundles and classifying spaces}\label{sec:bundles}

Standard references for this section are \cite{Hatcher, Husemoller, May, MS, Spanier}.

Let $G$ be a group with multiplication map $\mu : G \times G \to G$, $(x,y)\mapsto xy$ and inversion map $\iota: G\to G$, $x \mapsto x^{-1}$. If $G$ is also  a topological space such that $\mu$ and $\iota$ are continous then $G$ together with this topology is called a \textit{topological group}.
Every group $G$ is a topological group if we put the discrete topology on $G$; we will denote such a topological group by $G_d$ (unless the natural topology is the discrete topology, in which case we will just write $G$).  For example, $\mathbb{Z}$ with its natural discrete topology is a topological group.
In this article, we are primarily interested in the case where $G$ is the group of $2 \times 2$ real orthogonal matrices. When endowed with the manifold topology, this is $\SO(2)$, the special orthogonal group in dimension two and is homeomorphic to the unit circle $\mathbb{S}^1$ as a topological space. On the other hand, $\SO(2)_d$ is just a discrete uncountable collection of $2\times 2$ orthogonal matrices. Both $\SO(2)$ and $\SO(2)_d$ will be of interest to us.

Let $X,P,F$ be topological spaces. We say that $\pi: P \to X$ is a \textit{fiber bundle} with fiber $F$ and base space $X$ if $\pi$ is a continuous surjection and every point of $X$ has a neighborhood $U$ such that $\pi^{-1}(U)$ is homoeomorphic to $U\times F$.

In particular, $\pi^{-1}(x)\cong F$ for all $x \in X$.

A \textit{principal $G$-bundle} is a tuple $(P,\pi,\varphi)$ where $\pi: P \to X$ is a fiber bundle with fiber $G$ and $\varphi : G \times P \to P$ is a group action such that 
\begin{enumerate}[\upshape (i)]
\item $\varphi$ is a continuous map;
\item $\varphi(g,f) \in \pi^{-1}(x)$ for any $f\in \pi^{-1}(x)$;
\item if $\varphi(g,f) = f$ for some $f\in P$, then $g$ is the identity element in $G$;
\item For any $x$ and $f,f'\in \pi^{-1}(x)$, there is a $g\in G$ such that $\varphi(g,f) = f'$.
\end{enumerate}
We will often say `$P$ is a principal $G$-bundle on $X$' to mean the above, without specifying $\pi$ and $\varphi$.  A principal $\SO(2)$-bundle is called an \textit{oriented circle bundle} and a principal $\SO(2)_d$-bundle is called a \textit{flat oriented circle bundle}. We will have more to say about these in Sections~\ref{sec:main} and \ref{sec:main'}.

Let $(P,\pi,\varphi)$ and $(P',\pi',\varphi')$ be two principal $G$-bundles on $X$. We say that $(P,\pi,\varphi)$ is \textit{isomorphic} to $(P',\pi',\varphi')$, denoted $P\simeq P'$, if there is a homeomorphism $\vartheta :P\to P'$ compatible with the group actions $\varphi$, $\varphi'$ and  the projection maps $\pi$, $\pi'$ in the following sense:
\[
\vartheta \circ \varphi = \varphi' \circ (\operatorname{id}_G \times \vartheta)\quad \text{and}\quad
\pi' \circ \vartheta = \pi. 
\]
Here $\operatorname{id}_G : G \to G$ is the identity map. Let $\mathcal{U} = \{U_i : i\in I\}$ be an open covering of $X$ such that 
$\pi^{-1}(U_i) \cong U_i \times G$ via some isomorphism $\tau_i$ for all $i \in I$. A \textit{transition function}  corresponding to $\mathcal{U}$ is a map $\tau_{ij} \coloneqq \tau_i\tau_j^{-1}$, defined for all $i,j\in I$ such that $U_i\cap U_j\ne \varnothing$. It may be regarded as a $G$-valued function $\tau_{ij}:U_i\cap U_j \to G$. Transition functions are important because one may  construct a principal $G$-bundle entirely from its transition functions \cite{Husemoller}.

For $G = \SO(2)$, transition functions $\tau_{ij}$ of an oriented circle bundle are continuous $\SO(2)$-valued functions on open sets $U_i\cap U_j$. For $G = \SO(2)_d$, transition functions $\tau'_{ij}$ of a flat oriented circle bundle are continuous $\SO(2)_d$-valued functions on open sets $U_i\cap U_j$ but since $\SO(2)_d$ has the discrete topology, this means that $\tau'_{ij}$ are locally constant $\SO(2)$-valued functions on $U_i\cap U_j$. In particular, if $U_i\cap U_j$ is connected, then $\tau'_{ij}$ are constant $\SO(2)$-valued functions on $U_i\cap U_j$.  In our case, the covering that we choose (see \eqref{eq:open}) will have connected $U_i\cap U_j$'s and so we may regard
\begin{multline*}
\bigl\{\text{isomorphism classes of flat oriented circle bundles}\bigr\} \\
\subseteq  \bigl\{\text{isomorphism classes of oriented circle bundles}\bigr\}.
\end{multline*}
In other words, flat oriented circle bundles are just oriented circle bundles whose transition functions are constant-valued. 

Let $X,Y$ be topological spaces. Two maps $h_0,h_1:X\to Y$ are \textit{homotopic} if there is a continuous function $H:X\times I \to Y$ such that 
\[
H(x,0) = h_0(x) \qquad \text{and} \qquad H(x,1) = h_1(x).
\]
Homotopy is an equivalence relation and the set of homotopy equivalent classes of maps from $X$ to $Y$ is denoted by $[X,Y]$. Let  $\mathbb{S}^n$ be the $n$-sphere. We say that a topological space $X$ is \textit{weakly contractible} if $[\mathbb{S}^n,X]$ contains only the equivalence class of the \emph{trivial map}, i.e., the map that sends all points in $\mathbb{S}^n$ into a fixed point of $X$. The \textit{classifying space} of  a topological group $G$ is a topological space $BG$ together with a principal $G$-bundle $EG$ on $BG$ such that $EG$ is weakly contractible.

\begin{proposition}\label{prop:classifying space}
For any topological space $X$ and topological group $G$, there is a one-to-one correspondence between the following two sets:
\[
[X,BG] \longleftrightarrow \{\text{isomorphism classes of principal $G$-bundles on $X$}\},
\]
given by $h\mapsto h^{*}(EG)$, the principal $G$-bundle on $X$ whose fiber over $x\in X$ is the fiber of $EG$ over $h(x) \in BG$.
\end{proposition}

For the purpose of this paper, readers only need to know that the classifying space $B\U(n)$ of the unitary group $\U(n)$ is $\operatorname{Gr}(n,\infty)$, the Grassmannian of $n$-planes in $\mathbb{C}^{\infty}$. In particular, if $n=1$, since $\U(1) = \SO(2)$, we have 
\begin{equation}\label{eq:bu1}
B\SO(2) = \mathbb{C}P^{\infty}. 
\end{equation}

Let $G$ be an abelian group with identity $0$. We write $\operatorname{Hom}_\mathbb{Z} (G,\mathbb{Z})$ for the set of all homomorphisms from $G$ to $\mathbb{Z}$. An element $g\in G$ is a \textit{torsion} element if it has finite order, i.e., $g^n = 1$ for some $n \in \mathbb{N}$. The subgroup of all torsion elements in $G$ is called its \textit{torsion subgroup} and denoted $G_T$. For example, every element in $\mathbb{Z}/m\mathbb{Z}$ is a torsion element whereas $0$ is the only torsion element in $\mathbb{Z}$. For an abelian group $G$, we also denote its torsion subgroup as
\[
G_T = \operatorname{Ext}^1_\mathbb{Z}(G,\mathbb{Z}).
\]
The reason for including this alternative notation is that it is very standard ---  a special case of $\operatorname{Ext}$ groups for $G$ defined more generally \cite{Hatcher, Hungerford}.  We now state some routine relations \cite{Hungerford} that we will need for our calculations. Let $G$ and $G'$ be  abelian groups. Then
\begin{gather*}
\operatorname{Hom}_\mathbb{Z} (G_T,\mathbb{Z}) = 0,  \qquad
\operatorname{Hom}_\mathbb{Z} (G/G_T,\mathbb{Z}) \simeq G/G_T, \\
\operatorname{Hom}_\mathbb{Z} (G\oplus G',\mathbb{Z}) \simeq \operatorname{Hom}_\mathbb{Z} (G,\mathbb{Z}) \oplus \operatorname{Hom}_\mathbb{Z} (G',\mathbb{Z}),
\end{gather*}
and
\begin{gather*}
\operatorname{Ext}^1_{\mathbb{Z}}(G_T,\mathbb{Z}) = G_T,  \qquad
\operatorname{Ext}^1_{\mathbb{Z}}(G/G_T,\mathbb{Z}) = 0,  \\
\operatorname{Ext}^1_{\mathbb{Z}}(G\oplus G',\mathbb{Z})\simeq \operatorname{Ext}^1_{\mathbb{Z}}(G,\mathbb{Z})\oplus \operatorname{Ext}^1_{\mathbb{Z}}(G',\mathbb{Z}).
\end{gather*}
Singular homology and singular cohomology  are related via $\operatorname{Ext}^1_{\mathbb{Z}}$ and $\operatorname{Hom}_\mathbb{Z}$ in the following well-known theorem.
\begin{theorem}[Universal coefficient theorem]\label{thm:universal coefficient}
Let $X$ be a topological space. Then we have a natural short exact sequence
\[
0 \to \operatorname{Ext}^1_\mathbb{Z}(H_1(X),\mathbb{Z}) \to H^2(X)\to \operatorname{Hom}_\mathbb{Z}(H_2(X),\mathbb{Z})\to 0.
\]
In particular we have an isomorphism,
\[
H^2(X) \simeq \mathbb{Z}^{b} \oplus T_{1},
\] 
where $b \coloneqq \operatorname{rank} (H_2(X))= b_2(X)$ is the second Betti number of $X$ and $T_1$ is the torsion subgroup of $H_1(X)$.
\end{theorem}
The second Betti number of $X$ counts the number of $2$-dimensional `voids' in $X$. In the case of interest to us, where $X$ is a finite two-dimensional simplicial complex, the second Betti number counts the number of $2$-spheres (by which we meant the boundary of a $3$-simplex, which is homeomorphic to $\mathbb{S}^2$) contained in $X$.

We will also need the following alternative characterization  \cite[Chapter 22]{May} of $H^2(X)$.
\begin{theorem}\label{thm:K(Z,2) space}
Let $X$ be a topological space. Then we have
\[
[X, \mathbb{C}P^{\infty}] \simeq H^2(X).
\]
\end{theorem}

\subsection{{\v C}ech cohomology}\label{sec:cech}

Standard sources for this are \cite[Chapter $0$]{GH},\cite[Chapter $3$]{Hartshorne} and \cite[Chapter $2$]{Iversen}.

Let  $G$ be a topological abelian group and let $X$ be a topological space. For any open subset $U$ of $X$ we define an assignment
\[
U \mapsto \uline{G}(U)\coloneqq\text{group of $G$-valued continuous functions on $U$}
\]
for all open subset $U \subseteq X$. By definition, if $G$ is a discrete group and $U$ is any connected open subset of $X$, then $\uline{G}(U) = G$. If $U\subseteq V$ then we have a \textit{restriction map}
\[
\rho_{V,U}: \uline{G}(V) \to \uline{G}(U)
\]
defined by the restriction of $G$-valued continuous functions on $V$ to $U$.

Let $X$ be a topological space and $G$ be a topological abelian group on $X$. Let $\mathcal{U} = \{U_i : i\in I\}$ be an open covering of $X$. We may associate a cochain complex to $X$, $G$, and $\mathcal{U}$ as follows:
\begin{equation}\label{Cech cochain complex}
C^{0}(\mathcal{U},G) \xrightarrow{\delta_{0}}  C^{1}(\mathcal{U},G) \xrightarrow{\delta_{1}} C^{2}(\mathcal{U},G)
\end{equation}
where
\begin{align*}
C^{0}(\mathcal{U},G) &= \prod\nolimits_{i\in I} \uline{G}(U_i),  \\
C^{1}(\mathcal{U},G) &= \Bigl\{(g_{ij})_{i,j\in I} \in \prod\nolimits_{i,j \in I} \uline{G}(U_i\cap U_j): g_{ij}g_{ji} = 1\; \text{for all}\; i,j\in I\Bigr\}, \\
C^{2}(\mathcal{U},G) &= \Bigl\{(g_{ijk})_{i,j,k\in I}\in \prod\nolimits_{i,j,k \in I} \uline{G}(U_i\cap U_j\cap U_k):\\
&\qquad \qquad g_{ijk} g_{ikj}  = g_{ijk} g_{kji} = g_{ijk} g_{jik} = 1\; \text{for all}\; i,j,k\in I\Bigr\},
\end{align*}
and 
\[
\begin{aligned}
\bigl(\delta_0 (g_i)_{i \in I}\bigr)_{j,k}  &=g_{k}  g_{j}^{-1},  &&\text{for all}\; j,k \in I,\\
\bigl(\delta_1 (g_{ij})_{i,j \in I}\bigr)_{k,l,m} &= g_{lm}  g_{mk}  g_{kl} &&\text{for all}\; k,l,m \in I.
\end{aligned}
\]
To be precise, we have 
\begin{align*}
g_{k} g_{j}^{-1} &= \rho_{U_k,U_k\cap U_j} (g_k)  \cdot \rho_{U_j,U_k\cap U_j} (g_j^{-1}), \\
g_{lm}  g_{mk}  g_{kl} &= \rho_{U_l\cap U_m, U_k\cap U_l\cap U_m}(g_{lm})  \cdot \rho_{U_k\cap U_m, U_k\cap U_l\cap U_m}(g_{mk}) \cdot \rho_{U_k\cap U_l, U_l\cap U_m\cap U_k}(g_{kl}). 
\end{align*}
It is easy to check that $\delta_1 \circ \delta_0 = 0$ and so \eqref{Cech cochain complex}  indeed forms a cochain complex. 

As in the case of singular cohomology, $\check{B}^1(\mathcal{U},G) \coloneqq \operatorname{Im} \delta_0$ and  $\check{Z}^1(\mathcal{U},G) \coloneqq  \operatorname{Ker} \delta_1$ are the groups of \emph{\v{C}ech $1$-coboundaries} and \emph{\v{C}ech $1$-cocycles} respectively. Again we have $\check{B}^1(\mathcal{U},G)  \subseteq \check{Z}^1(\mathcal{U},G) $. The \textit{first {\v C}ech cohomology group associated to $\mathcal{U}$} with coefficients in $G$  is  then defined to be the quotient group
\[
\check{H}^1(\mathcal{U},G) \coloneqq \check{Z}^1(\mathcal{U},G) /\check{B}^1(\mathcal{U},G).
\]
Explicitly, we have 
\[
\check{H}^1(\mathcal{U}, G)  =\frac{\{(g_{ij}) : g_{ij} g_{jk}  g_{ki}  = 1 \;\textit{for all}\; i,j,k\}}{\{(g_{ij}) : g_{ij} = g_j  g_i^{-1} \;\textit{for all}\; i,j\}}.
\]
We have in fact already encountered this notion in Section~\ref{sec:tribar}, $\check{H}^1(\mathbb{R}^2, \mathbb{R}_+)$, the  {\v C}ech cohomology group of the plane $\mathbb{R}^2$ with coefficients in the group $\mathbb{R}_+$ has appeared implicitly in our discussion. 


By its definition, $\check{H}^1(\mathcal{U},G)$ depends on the choice of open covering $\mathcal{U}$ of $X$. To obtain a  {\v C}ech cohomology group of $X$ independent of open covering, we take the \textit{direct limit} over all possible open coverings of $X$. 
The \textit{first {\v C}ech cohomology group} of $X$ with coefficients in $G$ is defined to be the direct limit
\[
\check{H}^1(X,G) \coloneqq \varinjlim\check{H}^1(\mathcal{U},G)
\]
with $\mathcal{U}$ running through all open coverings of $X$. 

For those unfamiliar with the notion of direct limit, $\check{H}^1(X,G)$ may be defined explicitly using an equivalence relation:
\[
\check{H}^1(X,G) \coloneqq\Bigl[\coprod\nolimits_{\mathcal{U}} \check{H}^1(\mathcal{U},G) \Bigr]\!\!\Bigm/\sim,
\]
where $\coprod\nolimits_{\mathcal{U}}$ denotes the disjoint union of $\check{H}^1(\mathcal{U}, G)$ for all possible open coverings of $X$. The equivalence relation $\sim$ is given as follows: For $\varphi_{\mathcal{U}}\in \check{H}^1(\mathcal{U}, G)$ and $\varphi_\mathcal{V} \in \check{H}^1(\mathcal{V}, G)$,  $\varphi_{\mathcal{U}} \sim \varphi_{\mathcal{V}}$  iff
\begin{enumerate}[\upshape (i)]
\item there is an open covering $\mathcal{W}$ such that every open set $W\in \mathcal{W}$ is contained in $U\cap V$ for some $U\in \mathcal{U}$ and $V\in\mathcal{V}$;
\item there is  an element $\varphi_{\mathcal{W}}\in\check{H}^1(\mathcal{W}, G)$ such that the \emph{restriction} of $\varphi_{\mathcal{U}}$ and the restriction of $\varphi_{\mathcal{V}}$ are both equal to $\varphi_{\mathcal{W}}$.
\end{enumerate}
The term ``restriction'' needs elaboration. Let $\mathcal{U}=\{U_i : i\in I\}$, $\mathcal{V}=\{V_\alpha : \alpha\in \Lambda\}$ be open covers of $X$ such that for any $U_i \in \mathcal{U}$, there is some $V_{\alpha_i}\in\mathcal{V}$ with $U_i \subseteq V_{\alpha_i}$. Fix a map $\tau: I \to \Lambda$ such that $U_i \subseteq V_{\tau(i)}$. There is a natural restriction map $\rho_{\mathcal{V},\mathcal{U}}: \check{H}^1(\mathcal{V},G) \to \check{H}^1(\mathcal{U},G)$ induced by $\widetilde{\rho}_{\mathcal{V},\mathcal{U}}: C^1(\mathcal{V},G) \to C^{1}(\mathcal{U},G)$ where 
\[
(\widetilde{\rho}_{\mathcal{V},\mathcal{U}}(g_{\alpha,\beta}))_{i,j} = \rho_{V_{\tau(i)}\cap V_{\tau(j)},U_i\cap U_j} (g_{\tau(i),\tau(j)}).
\]
The image $\rho_{\mathcal{V},\mathcal{U}}(\varphi)$ of $\varphi\in \check{H}^1(\mathcal{V},G)$ is called the \emph{restriction} of $\varphi$ to $\check{H}^1(\mathcal{U},G)$. It does not depend on the choice of $\tau$.

As the reader can guess, calculating the  {\v C}ech cohomology group using such a definition would in general be difficult. Fortunately, the following theorem (really a special case of Leray's theorem \cite{Forster}) allows us to simplify the calculation in all cases of interest to us in this article.
\begin{theorem}[Leray's theorem]\label{thm:Leray}
Let $X$ be a topological space and  $G$ be an topological abelian group. Let $\mathcal{U} = \{U_i : i\in I\}$ be an open cover of $X$ such that $\check{H}^1(U_i,G) = 0$ for all $i\in I$. Then we have 
\[
\check{H}^1(\mathcal{U},G) \simeq \check{H}^1(X,G). 
\]
\end{theorem}

Furthermore, we will often be able to reduce calculation of {\v C}ech cohomology to calculation of singular cohomology since they are equal in the case when $X$ is a finite simplicial complex \cite{Prasolov}. 
\begin{theorem}\label{thm:singular=cech}
If $K$ is a finite simplicial complex and $G$ is an abelian group, then 
\[
\check{H}^1(K,G_d) \simeq H^1(K,G),
\]
where $G_d$ is the group $G$ equipped with the discrete topology.
\end{theorem}

For a contractible space, we have $H^1(K,G) = 0$ by Proposition~\ref{prop:contractible}. So we may deduce the following from Theorem~\ref{thm:singular=cech}.
\begin{corollary}\label{cor:vanishing}
If $K$ is a finite contractible simplicial complex and $G$ is an abelian group, then 
\[
\check{H}^1(K,G_d) = 0.
\]
\end{corollary}

To check whether an oriented circle bundle on a finite simplicial complex $K$ is flat, we have the following useful result \cite{KT,Morita,OT}.
\begin{proposition}\label{prop:torsion}
An oriented circle bundle on $K$ is flat if and only if its Euler class is a torsion element in $H^2(K)$.
\end{proposition}
Defining the Euler class of an oriented circle bundle would take us too far afield and so this will be the only term left undefined in our article. Fortunately, all we need is the following corollary of Proposition~\ref{prop:torsion}.
\begin{corollary}
If $H^2(K)$ is torsion free, then any oriented circle bundle on $K$ must be flat.
\end{corollary}

A particularly important result \cite{Brylinski,HJJS} for us is the following theorem that relates the {\v C}ech cohomology group with $G$-coefficients and  principal $G$-bundles.
\begin{theorem}\label{thm:classifying principal bundles}
If $G$ is a topological abelian group, then $\check{H}^1(X, G)$ is in canonical one-to-one correspondence with the set of isomorphism classes of principal $G$-bundles on $X$.
\end{theorem}

\section{Cohomological classification of discrete cryo-EM cocycles}\label{sec:main}

We will follow the mathematical setup for the cryo-EM problem as laid out in \cite{HS1, HS2}. First recall the high-level description of the problem: Given cocycles comprising a collection of noisy $2$D projected images, reconstruct the $3$D structure of the molecule that gave rise to these images. The \emph{standard mathematical model} for cryo-EM casts the problem in mathematical terms and may be described as follows:
\begin{enumerate}[\upshape (i)]
\item The \textit{molecule} is described by a function $\varphi : \mathbb{R}^3 \to \mathbb{R}$, the potential function of the molecule.
\item A \textit{viewing direction} is described by a point on the $2$-sphere $\mathbb{S}^2$.
\item The \textit{position} of an image is described by a $3 \times 3$ matrix $A= [a, b, c] \in \SO(3)$ where the orthonormal column vectors $a,b,c$ are such that $\operatorname{span} \{a, b\}$ is the projection plane and $c$ is the viewing direction.
\item A \textit{projected image} $\psi$ of the molecule $\varphi$ by $A$ is described by a  function $\psi : \mathbb{R}^2 \to \mathbb{R}$ where
\[
\psi(x,y) = \int_{z\in \mathbb{R}} \varphi(x a + y b + z c ) \, dz.
\]
The function $\psi$ describes the \textit{density} of the molecule along the chosen viewing direction.
\end{enumerate}
Let $\Psi = \{\psi_1,\dots, \psi_n\}$ be a set of $n$ projected images of the molecule and $c_1,\dots,c_n$ be the corresponding viewing directions. It is common to impose two mild assumptions:
\begin{enumerate}[\upshape (a)]
\item The function $\varphi$ is generic. In particular, each image $\psi_i\in \Psi$ has a uniquely determined viewing direction. In practice, this means that the molecule has no extra symmetry. This assumption does not exclude the possibility where two images $\psi_i, \psi_j$ may share the same viewing direction. However, it excludes the case where an image $\psi_i$ can be obtained from projections of the molecule from two different directions.
\item The viewing directions $c_1,\dots,c_n \in \mathbb{S}^2$ are  distributed uniformly on $\mathbb{S}^2$. This is a standard assumption in cryo-EM literature although in practice, viewing directions are rarely uniformly distributed.
\end{enumerate}
In addition, since each image $\psi_i$ is associated with a viewing direction $c_i$, we should regard $\psi_i$ to be a real-valued function on the tangent plane to $\mathbb{S}^2$ with unit normal in the direction of $c_i$. This is the point-of-view adopted in \cite{SZSH} and we will assume it throughout this article. An important distinction between cryo-EM and other reconstruction problems in medical imaging, remote sensing, underwater acoustics, etc, mentioned in Section~\ref{sec:intro} is that for the former, the viewing directions  $c_1,\dots,c_n$   are \emph{unknown} and have to be determined from the data set $\Psi$, whereas for the latter, we usually know in which directions the imaging instruments (CT scanner, camera, radar, sonar, etc) are pointed. In fact, determining $c_1,\dots,c_n$ from $\Psi$ is the most crucial step in cryo-EM --- our goal is to show that there is some interesting algebraic topology behind this problem.

Henceforth, by a `molecule,' we will mean one in the standard mathematical model, i.e., a function $\varphi$. These include $\varphi$'s that  do not correspond to any actual molecules.
We assume that $\varphi \in L^2(\mathbb{R}^3)$ and $\psi_1,\dots,\psi_n \in L^2(\mathbb{R}^2)$. There is a natural notion of distance  \cite{PZF} between projected images $\Psi = \{\psi_1,\dots, \psi_n\}$ given by
\[
d(\psi_i,\psi_j) =\min_{g\in \SO(2)} \lVert g \cdot \psi_i - \psi_j \rVert,
\]
where $\lVert\, \cdot \, \rVert$ is the norm in $L^2(\mathbb{R}^2)$ and the action of $g\in \SO(2)$ on a projected image $\psi$ is
\[
(g \cdot \psi)(x,y) = \psi (g^{-1} (x,y)).
\]
Geometrically, the action of $g$ on $\psi$ is the rotation of $\psi$ by the angle represented by $g\in \SO(2)$. Let $g_{ij}$ be the element in $\SO(2)$ which realizes the minimum of the distance $d(\psi_i,\psi_j)$, i.e.,
\begin{equation}\label{eq:gij}
g_{ij} \coloneqq \operatorname*{argmin}_{g\in \SO(2)} \, \lVert g \cdot \psi_i - \psi_j \rVert
\end{equation}
for $i,j = 1,\dots,n$. Clearly, we have
\begin{equation}\label{eq:reci}
g_{ii} = 1_n\qquad\text{and}\qquad g_{ij}g_{ji} = 1_n,
\end{equation}
for all $i,j=1,\dots,n$, where $1_n$ is the $n \times n$ identity matrix, which we will henceforth denote simply as $1$ when there is no cause for confusion. In general, $g_{ij}$ is not unique since it could happen that two different rotations both minimize the distance but our assumption that the function $\varphi$ is generic ensures that $g_{ij}$ is uniquely determined by $\psi_i$ and $\psi_j$. We will call
\begin{equation}\label{eq:data}
D \coloneqq \{g_{ij} \in \SO(2) : i,j =1,\dots,n\}
\end{equation}
the set of \textit{pairwise angular comparisons}. This is of course derived from the raw image data set $\Psi$ and the process of extracting $D$ from $\Psi$ is itself an active research topic \cite{BCS,BKS}, particularly when the images $\psi_i$'s are noisy. We will not concern ourselves with this auxiliary problem here.

We will use notations consistent with those introduced in Section~\ref{sec:sing} for simplices. For any $\varepsilon > 0$, we may construct an undirected graph $G_\varepsilon = (V,E)$ where $V= \{[1],\dots, [n]\}$ is the set of vertices corresponding to the projected images  $\Psi = \{\psi_1,\dots, \psi_n\}$, and  $E$ is  the set of edges defined by 
\begin{equation}\label{eq:rips}
[i,j] \in E  \quad \text{if and only if} \quad d(\psi_i,\psi_j) \le \varepsilon.
\end{equation}

Let us first consider an ideal situation where the projected images $\psi_i$'s  are noiseless. Also we fix $\varepsilon>0$ and the number of images $n$.
Let $G_\varepsilon$ be the associated undirected graph. We define the \textit{cryo-EM complex} $K_\varepsilon$ as follows:
\begin{enumerate}[\upshape (i)]
\item the $0$-simplices of $K_\varepsilon$ are the vertices of $G_\varepsilon$,
\item the $1$-simplices of $K_\varepsilon$ are the edges of $G_\varepsilon$,
\item the $2$-simplices of $K_\varepsilon$ are the triangles $[i,j,k]$ such that $[i,j], [i,k], [j,k]$ are all edges of $G_\varepsilon$.
\end{enumerate}
$K_\varepsilon$ is a two-dimensional finite simplicial complex. It is the \textit{$3$-clique complex} \cite{B,L} of the graph $G_\varepsilon$. In addition, $K_\varepsilon$  is also the \textit{Vietoris--Rips complex} \cite{DSG,Rips} defined by \eqref{eq:rips} with respect to the metric $d$.

Some simple examples: The graph $G_1=(V_1,E_1)$ with $V_1=\{[1], [2],[3]\}$ and $E_1=\{[1,2],[1,3],[2,3]\}$ defines a simplicial complex $K_1$ that is a triangle.
The graph $G_2=(V_2,E_2)$ with $V_2=\{[1], [2],[3],[4]\}$ and $E_2=\{[1,2],[1,3],[2,3],[1,4],[2,4],[3,4]\}$ defines a simplicial complex $K_2$ that is the boundary of a tetrahedron or $3$-simplex. The graph $G_3=(V_3,E_3)$ with $V_3=\{[1], [2],[3],[4]\}$ and $E_3=\{[1,2],[2,3],[1,4],[3,4]\}$ defines a simplicial complex $K_3$ that is the boundary of a square.
\[
K_1 \hspace{-4ex}
\begin{tikzpicture}
\node (a) at (0,-0.3) {1};
\node (b) at (1.3,1) {2};
\node (c) at (-1.3,1) {3};
\draw (0,0)--(-1,1);
\draw (-1,1)--(1,1);
\draw (1,1)--(0,0);
\filldraw (0,0) circle (1pt);
\filldraw (1,1) circle (1pt);
\filldraw (-1,1) circle (1pt);
\end{tikzpicture}\qquad\qquad
K_2 \hspace{-4ex}
\begin{tikzpicture}
\node (a) at (0,-0.3) {1};
\node (b) at (1.3,1) {2};
\node (c) at (0,2.3) {3};
\node (d) at (-1.3,1) {4};
\draw (0,2)--(-1,1);
\draw (0,2)--(1,1);
\draw (0,2)--(0,0);
\draw[dotted] (-1,1)--(1,1);
\draw (0,0)--(-1,1);
\draw (1,1)--(0,0);
\filldraw (0,0) circle (1pt);
\filldraw (1,1) circle (1pt);
\filldraw (-1,1) circle (1pt);
\filldraw (0,2) circle (1pt);
\end{tikzpicture}\qquad\qquad
K_3 \hspace{-4ex}
\begin{tikzpicture}
\node (a) at (0,-0.3) {1};
\node (b) at (1.3,1) {2};
\node (c) at (0,2.3) {3};
\node (d) at (-1.3,1) {4};
\draw (0,2)--(-1,1);
\draw (0,2)--(1,1);
\draw (0,0)--(-1,1);
\draw (1,1)--(0,0);
\filldraw (0,0) circle (1pt);
\filldraw (1,1) circle (1pt);
\filldraw (-1,1) circle (1pt);
\filldraw (0,2) circle (1pt);
\end{tikzpicture}
\]
We will regard our simplicial complex $K_\varepsilon$ as being embedded in $\mathbb{R}^4$ and inherits the Euclidean topology from $\mathbb{R}^4$, i.e., $K_\varepsilon$ is a geometric simplicial complex and not just an abstract simplicial complex. For each vertex $[i]$ of $K_\varepsilon$ we define an open set $U_i(K_\varepsilon)$ to be the union of the interior of all simplices of $K_\varepsilon$ containing the vertex $[ i]$. Those familiar with simplicial complex might like to note that $U_i(K_\varepsilon)$ is just the complement of the \textit{link} of $[ i]$ in the \textit{star} of $[ i]$. 
For example, $U_1(K_i)$ for $i=1,2,3$ are shown below. Here dashed lines are excluded from the neighborhood.
\[
U_1(K_1) \hspace{-4ex}
\begin{tikzpicture}
\node (a) at (0,-0.3) {1};
\node (b) at (1.3,1) {2};
\node (c) at (-1.3,1) {3};
\draw (0,0)--(-1,1);
\draw[thick,dashed] (-1,1)--(1,1);
\draw (1,1)--(0,0);
\filldraw (0,0) circle (1pt);
\end{tikzpicture}\qquad
U_1(K_2) \hspace{-4ex}
\begin{tikzpicture}
\node (a) at (0,-0.3) {1};
\node (b) at (1.3,1) {2};
\node (c) at (0,2.3) {3};
\node (d) at (-1.3,1) {4};
\draw[thick,dashed] (0,2)--(-1,1);
\draw[thick,dashed] (0,2)--(1,1);
\draw (0,2)--(0,0);
\draw[thick,dashed] (-1,1)--(1,1);
\draw (0,0)--(-1,1);
\draw (1,1)--(0,0);
\filldraw (0,0) circle (1pt);
\end{tikzpicture}\qquad
U_1(K_3) \hspace{-4ex}
\begin{tikzpicture}
\node (a) at (0,-0.3) {1};
\node (b) at (1.3,1) {2};
\node (c) at (0,2.3) {3};
\node (d) at (-1.3,1) {4};
\draw[thick,dashed] (0,2)--(-1,1);
\draw[thick,dashed] (0,2)--(1,1);
\draw (0,0)--(-1,1);
\draw (1,1)--(0,0);
\filldraw (0,0) circle (1pt);
\end{tikzpicture}
\]
It follows from our definition of $U_i(K_\varepsilon)$ that
\begin{equation}\label{eq:open}
\mathcal{U}=\{U_i : [i]\; \text{is a vertex of}\; K_\varepsilon\}
\end{equation}
is an open covering of $K_\varepsilon$.

Let $\varphi$ be a fixed molecule and $\Psi = \{\psi_1,\dots, \psi_n\}$ be a set of projected images of $\varphi$. 
The set of pairwise angular comparisons $D= \{g_{ij} \in \SO(2) : i,j =1,\dots,n\}$ contains all $g_{ij}$'s corresponding to every pair of images $\psi_i,\psi_j$.
For the purpose of cryo-EM reconstruction,  one does not usually need all elements in the $D$ \cite{SZSH}, only a much smaller subset  comprising the $g_{ij}$'s corresponding to images $\psi_i,\psi_j$ that are near each other, i.e., $d(\psi_i,\psi_j) \le \varepsilon$ for some small $\varepsilon > 0$. This is expected since most reconstruction methods proceed by aggregating \textit{local} information.
With this in mind, we define the following.
\begin{definition}
Let $D = \{g_{ij} \in \SO(2) : i,j =1,\dots,n\}$ be the set of pairwise angular comparisons. Let $\varepsilon > 0$ and $K_\varepsilon$ be the cryo-EM complex.
The \textit{discrete cryo-EM cocycle on $K_\varepsilon$} is the  subset of $D$ corresponding to edges in $K_\varepsilon$ given by
\[
z^d_\varepsilon \coloneqq \{g_{ij}\in \SO(2) : [i,j]\in K_\varepsilon \}.
\]
\end{definition}
We may view $z_\varepsilon^d$ as the `useful' part of the set of pairwise angular comparisons $D$ for cryo-EM reconstruction. In fact we are unaware of any reconstruction method that makes use of $g_{ij}$ where $[i,j]\notin K_\varepsilon$.

As we mentioned earlier in this section, we take the point-of-view in \cite{SZSH} that the projected images $\psi_i$'s lie in tangent planes of a two-sphere determined by their viewing directions. We also assume, as in \cite{SZSH}, that if the images $\psi_i$, $\psi_j$, and $\psi_k$ have viewing directions close enough, then they lie in the \textit{same} tangent plane. This assumption is reasonable since if $\psi_i$ and $\psi_j$  share the same viewing direction, then they will only differ by a  plane rotation. Moreover, if $\psi_i$, $\psi_j$ and $\psi_k$ share the same viewing direction, then the angle needed to rotate $\psi_i$ to $\psi_k$ is the sum of the angle needed to rotate $\psi_i$ to $\psi_j$ and the angle needed to rotate $\psi_j$ to $\psi_k$ --- implying that the $g_{ij}$'s corresponding to $\Psi = \{\psi_1,\dots,\psi_n\}$ satisfy  the following \textit{$1$-cocycle condition}:
\begin{equation}\label{eq:cocy}
g_{ij} g_{jk} g_{ki} = 1.
\end{equation}
Here $1$ is the identity matrix in $\SO(2)$. Note that the matrices $g_{ij}$'s in the discrete cryo-EM cocycle always satisfy \eqref{eq:reci}, irrespective of whether viewing directions are close enough.

By the preceding discussion, we will \emph{assume} that for $\varepsilon>0$ small enough, the $g_{ij}$'s will  satisfy the $1$-cocycle condition \eqref{eq:cocy} for all edges  $[i,j],[j,k],[k,i]$ of the cryo-EM complex $K_\varepsilon$. One motivation for this assumption is that when $\varepsilon \to 0$, images that lie in an $\varepsilon$-neighborhood will share the same viewing direction and thus $g_{ij}$'s will satisfy the cocycle condition \eqref{eq:cocy}. Therefore, ``small enough $\varepsilon$" should be taken mathematically to mean the value of $\varepsilon$ such that \eqref{eq:cocy} holds, bearing in  mind that \eqref{eq:cocy}, like any mathematical model, is ultimately only an approximation of reality. Our assumption that the $1$-cocycle condition is satisfied for small enough $\varepsilon >0$ is a basic tenet for our subsequent discussions. As far as we know, this assumption is not in existing cryo-EM literature although it is closely related to the ``same tangent plane'' assumption in \cite{SZSH}. While never explicitly stated, \eqref{eq:cocy} is the implicit principle underlying many, if not most, denoising techniques for cryo-EM images \cite{SW1,SZSH,Singer}, as we will see in Section~\ref{sec:denoise}.

Given an open subset $U$ of $K_\varepsilon$, any element $g\in \SO(2)$ can be regarded as the constant $\SO(2)$-valued function sending every point $x\in U$ to $g$, and thus we may regard $z_\varepsilon^d$ as a cocycle in $\check{Z}^1\bigl(K_\varepsilon,\SO(2)_d\bigr)$. We highlight this observation as follows:
\begin{center}
\textit{Every discrete cryo-EM cocycle on $K_\varepsilon$  is an $\SO(2)_d$-valued {\v C}ech $1$-cocycle on $K_\varepsilon$.}
\end{center}
Henceforth  we will regard
\[
\check{Z}^1\bigl(K_\varepsilon,\SO(2)_d\bigr)  = \bigl\{\text{all discrete cryo-EM cocycles on}\; K_\varepsilon\bigr\}.
\]
The set on the right includes all possible discrete cryo-EM cocycles on $K_\varepsilon$ corresponding to all molecules $\varphi$. A cocycle $z^d_\varepsilon$ only tells us how to glue together local information. It is possible for two different $3$D molecules to give the same discrete cryo-EM cocycle $z^d_\varepsilon$ as long as the relations between their projected images are the same.

Given a discrete cryo-EM cocycle $z^d_\varepsilon \in \check{Z}^1\bigl(K_\varepsilon,\SO(2)_d\bigr) $, i.e., elements in $z^d_\varepsilon$ satisfy \eqref{eq:cocy}, and any arbitrary image $\psi \in L^2(\mathbb{R}^2)$, we may apply each $g \in z^d_\varepsilon$ to $\psi$ to obtain a set of images
\[
z^d_\varepsilon(\psi) \coloneqq \{ g \cdot \psi : g \in z^d_\varepsilon \} = \{ g_{ij} \cdot \psi : [i,j] \in K_\varepsilon \}.
\]
The cocycle condition  \eqref{eq:cocy} ensures that for any image $g \cdot \psi$ in this set, we obtain the same set of images by applying each $g \in z^d_\varepsilon$, i.e.,
\[
z^d_\varepsilon(g \cdot \psi)  = z^d_\varepsilon(\psi)\qquad \text{for any} \; g \in z^d_\varepsilon.
\]
Moreover, the discrete cryo-EM cocycle obtained would be exactly $z^d_\varepsilon$. A set of projected images $z^d_\varepsilon(\psi)$ allows one to reconstruct the $3$D molecule $\varphi$ whose projected images are precisely the ones in  $z^d_\varepsilon(\psi)$ \cite{Frank1,Frank2,Radermacher,PRF}. Put in another way,  given a discrete cryo-EM cocycle $z^d_\varepsilon\in \check{Z}^1\bigl(K_\varepsilon,\SO(2)_d\bigr)$ and an image $\psi \in L^2(\mathbb{R}^2)$, we may construct a $3$D molecule $\varphi \in L^2(\mathbb{R}^3)$ whose discrete cryo-EM cocycle is exactly $z^d_\varepsilon$ and one of whose projected image is $\psi$.

The context for the following theorem is that we are given two collections of $n$ projected images $\Psi=\{\psi_1,\dots,\psi_n\}$ and $\Psi'=\{\psi'_1,\dots,\psi'_n\}$ of the same molecule $\varphi$. These give two discrete cryo-EM cocycles $D = \{g_{ij} \in \SO(2) : i,j =1,\dots,n\}$ and $D' = \{g'_{ij} \in \SO(2) : i,j =1,\dots,n\}$.  Let $\varepsilon >0$ be sufficiently small and $z^d_\varepsilon = \{g_{ij}\in \SO(2) : [i,j]\in K_\varepsilon \}$, ${z'}_\varepsilon^d = \{g'_{ij}\in \SO(2) : [i,j]\in K_\varepsilon \}$ be the corresponding discrete cryo-EM cocycles on $K_\varepsilon$.

\begin{theorem}[Bundle classification of discrete cryo-EM cocycles]\label{thm:class}
Let $\varepsilon > 0 $ be small enough so that \eqref{eq:cocy} holds and let $K_\varepsilon$ be the corresponding cryo-EM complex. Then
\begin{enumerate}[\upshape (i)]
\item  the $1$-cocycle $z^d_\varepsilon$ determines a flat oriented circle bundle on $K_\varepsilon $;
\item two $1$-cocycles $z^d_\varepsilon$ and $z^{\prime d}_\varepsilon$ for the same molecule determine isomorphic flat oriented circle bundles if and only if 
\begin{equation}\label{eq:cross1}
g'_{ij} = g_{ij} g_i g_j^{-1}
\end{equation}
for some $g_i,g_j\in \SO(2)$, $[i,j] \in K_\varepsilon$.
\end{enumerate}
\end{theorem}
\begin{proof}
Let $\mathcal{U} = \{U_i(K_\varepsilon) : i = 1,\dots,n\}$ be the open cover defined in \eqref{eq:open}. It is easy to see that $ U_i(K_\varepsilon)$ is contractible and so by Corollary~\ref{cor:vanishing},
\[
\check{H}^1\bigl(U_i(K_\varepsilon),\SO(2)_d)\bigr) = \{1\}
\]
for all $i =1,\dots, n$. We may then apply Theorem~\ref{thm:Leray} to get 
\[
\check{H}^1\bigl(\mathcal{U},\SO(2)_d)\bigr) \simeq \check{H}^1\bigl(K_\varepsilon,\SO(2)_d\bigr).
\]
Therefore it follows from Theorem~\ref{thm:classifying principal bundles} that $\check{H}^1\bigl(\mathcal{U}, \SO(2)_d\bigr)$ is canonically in one-to-one correspondence with the set of isomorphism classes of $\SO(2)_d$-principal bundles, i.e., flat oriented circle bundles. Since the subset $z^d_\varepsilon =\{ g_{ij} \in \SO(2) : [i,j] \in K_\varepsilon\}$ is a $1$-cocycle in $\check{H}^1(\mathcal{U}, \SO(2)_d)$, it determines an oriented circle bundle over $K_\varepsilon$. Part (ii) follows from the fact that the $1$-cocycle $b_\varepsilon = \{g_ig_j^{-1} \in \SO(2) : [i,j]\in K_\varepsilon \}$ is a $1$-coboundary and thus represents the trivial cohomology class.
\end{proof}

If the reader finds \eqref{eq:cross1} familiar, that is because we have seen a similar version \eqref{eq:cross} in our discussion of the Penrose tribar. The difference here is that the quantities in \eqref{eq:cross} are from the group $\mathbb{R}_+$ whereas the quantities in \eqref{eq:cross1} are from the group $\SO(2)$. Two cocycles $z^d_\varepsilon =\{ g_{ij} \in \SO(2) : [i,j] \in K_\varepsilon\} $ and $z^{\prime d}_\varepsilon =\{ g'_{ij} \in \SO(2) : [i,j] \in K_\varepsilon\} $ are said to be \textit{cohomologically equivalent} if and only if they differ by a coboundary $b_\varepsilon = \{g_ig_j^{-1} \in \SO(2)  : [i,j]\in K_\varepsilon \}$ in the sense of  \eqref{eq:cross1}.  Cohomologically equivalent $z^d_\varepsilon$ and ${z'}^d_\varepsilon$ define the same \textit{cohomology class} in the quotient group and we have
\begin{align*}
\check{H}^1\bigl(K_\varepsilon,\SO(2)_d\bigr) &\coloneqq \check{Z}^1\bigl(K_\varepsilon,\SO(2)_d\bigr) / \check{B}^1\bigl(K_\varepsilon,\SO(2)_d\bigr)\\
&= \bigl\{\text{cohomologically equivalent discrete cryo-EM cocycles on}\; K_\varepsilon \bigr\}.
\end{align*}

By Proposition~\ref{prop:classifying space}, the cohomology group $\check{H}^1\bigl(K_\varepsilon,\SO(2)_d\bigr)$ can be identified as sets with the classifying space
$[K_\varepsilon, B\SO(2)_d]$,
which classifies the isomorphism classes of flat oriented circle bundles on $K_\varepsilon$. We obtain  a canonical one-to-one correspondence 
\begin{multline}\label{eq:1-1 flat}
\bigl\{\text{cohomologically equivalent discrete cryo-EM cocycles on}\; K_\varepsilon \bigr\} \\ \longleftrightarrow \bigl\{\text{isomorphism classes of flat oriented circle bundles on}\; K_\varepsilon \bigr\}.
\end{multline}

Finally we arrive at the following result.
\begin{theorem}\label{thm:flat=trivial}
Let $\varepsilon > 0 $ be small enough so that \eqref{eq:cocy} holds and let $K_\varepsilon$ be the corresponding cryo-EM complex. Then
\begin{enumerate}[\upshape (i)]
\item every flat oriented circle bundle on $K_\varepsilon$ is the trivial circle bundle;
\item all discrete cryo-EM cocycles on $K_\varepsilon$ are coboundaries $b_\varepsilon = \{g_ig_j^{-1}\in \SO(2): [i,j]\in K_\varepsilon\}$.
\end{enumerate}
\end{theorem}
\begin{proof}
By  Proposition~\ref{prop:torsion}, it suffices to show that $H^2(K_\varepsilon)$ is torsion free.  But this follows from Theorem~\ref{thm:universal coefficient}: By our construction of $K_\epsilon$,  the simplicial complex is actually homotopic to a one-point union of several spheres or a one-point union of several circles. This implies that either $H_1(K_\epsilon) = 0$ or $H_1(K_\epsilon) \simeq \mathbb{Z}^r$ for some integer $r\ge 1$. In particular, $H_1(K_\epsilon)$ is torsion free, i.e., $T_1 =0$.
\end{proof}
In other words, the set on the right of \eqref{eq:1-1 flat} is a singleton comprising only the trivial bundle.
Consequently, discrete cryo-EM cocycles on $K_\varepsilon$ are all cohomologically equivalent and all correspond to the trivial circle bundle. So Theorem~\ref{thm:class} does not provide an interesting classification. The reason is that a  discrete cryo-EM cocycle as defined by \eqref{eq:gij}, i.e., an element of $\check{H}^1\bigl(K_\varepsilon,\SO(2)_d\bigr)$, is too coarse. In the next section, we will see how the classification becomes more interesting mathematically when we look at \textit{continuous} cryo-EM cocycles.

\section{Cohomological classification of continuous cryo-EM cocycles}\label{sec:main'}

In the standard mathematical model for cryo-EM, a projected image is a function $\psi:\mathbb{R}^2 \to \mathbb{R}$ defined by 
\[
\psi(x,y) = \int_{z\in \mathbb{R}} \varphi(xa + yb +zc) \, dz,
\]
where $A = [a,b,c] \in \SO(3)$ describes the orientation of the molecule in $\mathbb{R}^3$ and $\varphi$ is the potential function of the molecule. For every pair of images $\psi_i,\psi_j$ we define an $\SO(2)$-valued function 
\begin{equation}\label{eqn:hij}
h_{ij}(r) \coloneqq \operatorname{argmin}_{g\in \SO(2)} \int_{0}^{2\pi} | (g\cdot \psi_i)(r\cos\theta,r\sin\theta) - \psi_j(r\cos\theta,r\sin\theta) |^2 \, d\theta,
\end{equation}
where $r = \sqrt{x^2+ y^2}$. These $h_{ij}$'s should be interpreted as follows: We regard a 2D image $\psi_i$ as comprising circular `slices' of different radii as in Figure~\ref{fig:definition of h_ij}, i.e., each slice is the intersection of the  image $\psi_i$ with a circle of radius $r$. For each pair $i,j$,  $h_{ij}(r) \in \SO(2)$ is the rotation that minimizes the difference between the slice of $\psi_i$ of radius $r$ and the slice of $\psi_j$ of radius $r$.
\begin{figure}[h]
\begin{tikzpicture}
    \color{gray!50!white}
        \pgftransformrotate{-90}
    \physicalpendulum
    
\draw[blue!] (0,-4) circle (3cm);
\draw[blue!] (0,-4) circle (2cm);
\draw[blue!] (0,-4) circle (1cm);
\draw[blue!] (0,-4) -- (0,-1);
\filldraw[blue] (0,-4) circle (0.05cm);
\filldraw[blue] (0,-3) circle (0.05cm);
\filldraw[blue] (0,-2) circle (0.05cm);
\filldraw[blue] (0,-1) circle (0.05cm);

\node[mark size=5pt,black!]  at (0,0) {$\psi_i$};
\node[mark size=5pt,blue!]  at (0.35,-0.75) {$r_3$};
\node[mark size=5pt,blue!]  at (0.35,-1.75) {$r_2$};
\node[mark size=5pt,blue!]  at (0.35,-2.75) {$r_1$};

\end{tikzpicture}
\label{fig:definition of h_ij}
\caption{Circular slices of an image $\psi_{i}$.}
\end{figure}
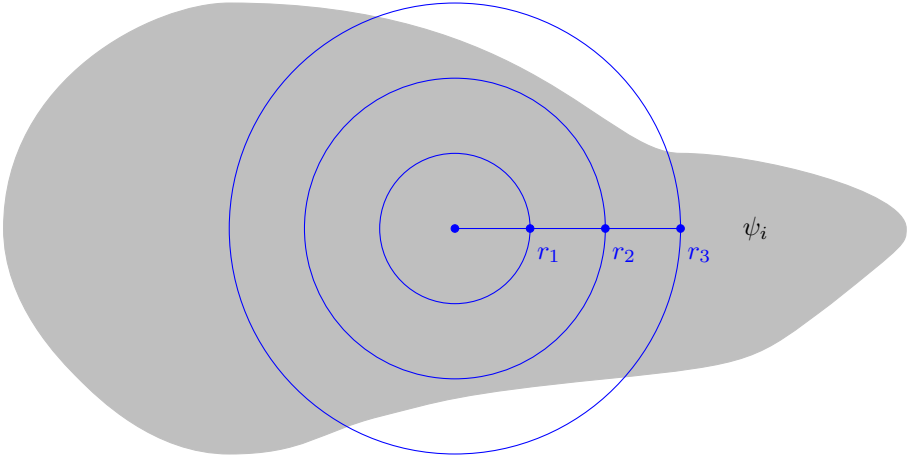

The integral in \eqref{eqn:hij} is in fact a restriction of the \emph{circular Radon transform} \cite{John}, defined for a compactly supported $f : \mathbb{R}^2 \to \mathbb{R}$ by
\[
Sf: \mathbb{R}^2 \times (0,\infty) \to \mathbb{R},\qquad
Sf(\eta,\xi,r) = \int_{ (x-\eta)^2 + (y - \xi)^2 = r} f(x,y) \, d\sigma(x,y),
\]
where $\sigma$ denotes the surface measure on the circle of radius $r$ centered at $(\eta,\xi)$.
Although it has, as far as we know, not been used in cryo-EM applications, the circular Radon transform is common in a variety of other applications, e.g., thermoacoustic tomography and optoacoustic tomography \cite{AK,FPR,KLFA,XW,AK1}.  If we set $(\eta,\xi) = (0,0)$, $x= r\cos \theta$, $y =r \sin \theta$, then
\[
Sf(0,0,r) = \int_{0}^{2\pi} f(r\cos\theta,r\sin\theta) \, d\theta,
\]
and so \eqref{eqn:hij} is the circular Radon transform of $|g \cdot \psi_i - \psi_j|^2$ at $(0,0,r)$.

Let $\varepsilon>0$ and the potential function $\varphi:\mathbb{R}^3 \to \mathbb{R}$ be chosen so that $h_{ij}(r)$ satisfies the $1$-cocycle condition
\begin{equation}\label{eqn:cocycle h_ij}
h_{ij}(r) h_{jk}(r) h_{ki}(r) = 1
\end{equation}
for all $r > 0$ whenever the images $\psi_i$, $\psi_j$, and $\psi_k$ are such that 
\[
d(\psi_i,\psi_j) \le \varepsilon, \qquad d(\psi_j,\psi_k) \le \varepsilon,\qquad d(\psi_k,\psi_i) \le \varepsilon.
\]
We remind readers that the existence of such an $\varepsilon$ that guarantees \eqref{eqn:cocycle h_ij} is an underlying basic tenet of our model.
Since  $\varphi$ is compactly supported, so must its projections $\psi_i$'s, implying that $h_{ij}$ is eventually constant, i.e., there exists some $R > 0$ and some $g\in \SO(2)$ such that $h_{ij}(r) = g$ whenever $r \ge R$. In fact there is no loss of generality in assuming that $g = 1$:  Since $\SO(2)$ is connected, we may pick a continuous curve $\gamma:[R,R'] \to \SO(2)$ such that $\gamma(R) = g$ and $\gamma(R') = 1$; replacing $h|_{[R,R']}$ by $\gamma$ then gives an $h$ where $h_{ij}(r) = 1$ for sufficiently large $r$.  In particular, $\lim_{r\to \infty} h_{ij}(r) = 1$, the identity element in $\SO(2)$.

Recall that we write $\uline{G}(U)$ for the set of $G$-valued functions on an open set $U$. So  $h_{ij}\in \uline{\SO(2)}(\mathbb{R}^2)$. 
Let $\mathcal{U}$ be the open covering  of $K_\varepsilon$ in \eqref{eq:open}. We will now define a \textit{continuous cryo-EM cocycle}, a {\v C}ech $1$-cocycle
\[
z^c_\varepsilon\coloneqq \{\tau_{ij}\in \uline{\SO(2)}(U_i\cap U_j): [i,j]\in K_\varepsilon\}.
\]
on $K_\varepsilon$ determined by the $h_{ij}$'s. The process is analogous to how we obtained  $z^d_\varepsilon$, the discrete cryo-EM cocycle on $K_\varepsilon$,  from the set of pairwise angular comparisons $D$ in Section~\ref{sec:main} but is a little more involved.

We first define the restriction of $\tau_{ij}$ to $U_i\cap U_j \cap U_k$ for all $k=1,\dots,n$ and show that we can glue them together to obtain a globally defined $\SO(2)$-valued function on $U_i\cap U_j$. By construction, the open covering $\mathcal{U}$ has the property that for any $U_i$, $U_j$, $U_k$, either
\[
U_i\cap U_j\cap U_k = \varnothing \qquad \text{or} \qquad U_i \cap U_j \cap U_k \cong \mathbb{R}^2.
\]
In the first case there is nothing to define. If $U_i \cap U_j \cap U_k \cong \mathbb{R}^2$, we fix a homeomorphism and regard $U_i \cap U_j \cap U_k$ as $\mathbb{R}^2$, then define the restriction of $\tau_{ij}$ to be 
\[
\tau_{ij}(x,y) = h_{ij}(r),
\]
for $(x,y)\in U_i \cap U_j \cap U_k$ and $h_{ij}\in \uline{\SO(2)}(U_i\cap U_j)$. Although the definition of $\tau_{ij}|_{U_i \cap U_j \cap U_k}$ depends on a homeomorphism $U_i \cap U_j \cap U_k \cong \mathbb{R}^2$, two such homeomorphisms induce a homeomorphism from $U_i \cap U_j \cap U_k $ to itself. So we obtain a one-to-one correspondence between the set of $\tau_{ij}$'s constructed from one homeomorphism and the set of $\tau_{ij}$'s constructed from the other. This in turn induces a one-to-one correspondence between cohomology classes represented by the two sets of $\tau_{ij}$'s. So while different homeomorphisms $U_i \cap U_j \cap U_k \cong \mathbb{R}^2$ give different $\tau_{ij}$'s, their cohomology classes are in one-to-one correspondence.

Since $U_i\cap U_j\cap U_k$ is disjoint from $U_i\cap U_j \cap U_l$ whenever $k$ and $l$ are distinct, to define $\tau_{ij}$ on $U_i\cap U_j$ we only need to define it on the set
\[
V_{ij} \coloneqq U_i\cap U_j - \bigcup\nolimits_{k\ne i,j} U_i\cap U_j \cap U_k.
\]
If $V_{ij} \ne \varnothing$,
then it must be the interior of the $1$-simplex connecting $[i]$ and $[j]$. In this case we define $\tau_{ij}$ to be the constant $\lim_{r\to \infty} \tau_{ij}(x,y) = 1\in \SO(2)$ where $(x,y)\in U_i\cap U_j \cap U_k$ and $r = \sqrt{x^2+y^2}$. Note that the limit exists as $\tau_{ij}(x,y)$ depends only on $r =\sqrt{x^2+y^2}$ and $\varphi$ and $\psi_i$'s are compactly supported. Lastly, it is obvious from its definition that $\tau_{ij}$ satisfies the $1$-cocyle condition
\begin{equation}\label{eqn:cocycle tau_ij}
\tau_{ij}(x,y) \tau_{jk}(x,y) \tau_{ki}(x,y) = 1.
\end{equation}

To illustrate our construction of $\tau_{ij}$,  we consider an example where the two-dimensional simplicial complex $K$ is obtained by glueing two triangles as follows:
\[
\begin{tikzpicture}
\node (a) at (0,-0.3) {$j$};
\node (b) at (1.3,1) {$l$};
\node (c) at (0,2.3) {$i$};
\node (d) at (-1.3,1) {$k$};
\node at (-0.5,1) {$h_{ij}$};
\node at (0.5,1) {$h_{ij}$};
\draw (0,2)--(-1,1);
\draw (0,2)--(1,1);
\draw (0,0)--(-1,1);
\draw (1,1)--(0,0);
\draw (0,0)--(0,2);
\filldraw (0,0) circle (1pt);
\filldraw (1,1) circle (1pt);
\filldraw (-1,1) circle (1pt);
\filldraw (0,2) circle (1pt);
\end{tikzpicture}
\]
Here $U_i\cap U_j\cap U_k$ is the interior of the triangle with vertices $i,j,k$. We define the values of $\tau_{ij}$, $\tau_{ki}$, and $\tau_{jk}$ on $U_i\cap U_j\cap U_k$ to be $h_{ij}$,  $h_{ki}$, and $h_{jk}$ respectively. One should think of $U_i\cap U_j\cap U_k$ as a copy of $\mathbb{R}^2$ and the boundary of the triangle with vertices $i,j,k$ as the ``points at infinity'' of $\mathbb{R}^2$.

Since $z^c_\varepsilon$ satisfies \eqref{eqn:cocycle tau_ij}, we see that $z^c_\varepsilon\in \check{Z}^1\bigl(K_\varepsilon,\SO(2)\bigr)$. By an argument similar to the proof of Theorem \ref{thm:class}, we obtain the following classification result.
\begin{theorem}[Bundle classification of continuous cryo-EM cocycles I]\label{thm:class'}
Let $\varepsilon > 0 $ be small enough so that \eqref{eqn:cocycle tau_ij} holds and let $K_\varepsilon$ be the corresponding cryo-EM complex. Then
\begin{enumerate}[\upshape (i)]
\item  the $1$-cocycle $z^c_\varepsilon$ determines an oriented circle bundle on $K_\varepsilon $;
\item two $1$-cocycles $z^c_\varepsilon$ and $z^{\prime c}_\varepsilon$ for the same molecule determine isomorphic oriented circle bundles if and only if 
\begin{equation}\label{eq:cross1'}
\tau'_{ij} = \tau_{ij} \tau_i \tau_j^{-1}
\end{equation}
for some $\tau_i \in \uline{\SO(2)}(U_i),\tau_j \in \uline{\SO(2)}(U_j)$, $[i,j] \in K_\varepsilon$.
\end{enumerate}
\end{theorem}

For small enough $\varepsilon>0$, Theorem~\ref{thm:class'} gives us a classification of all possible continuous cryo-EM cocycles on $K_\varepsilon$, a canonical correspondence
\begin{multline}\label{eq:bunclass}
\bigl\{\text{cohomologically equivalent cryo-EM cocycles on}\; K_\varepsilon \bigr\}\\
 \longrightarrow \bigl\{\text{isomorphism classes of oriented circle bundles on}\; K_\varepsilon \bigr\}.
\end{multline}
By Proposition~\ref{prop:classifying space}, the  isomorphism classes of principal $G$-bundles may be identified with $[K_\varepsilon, BG]$, the homotopy classes of continuous maps from $K_\varepsilon$ to the classifying space of $G$. In our case, $G=\SO(2) \simeq \mathbb{S}^1$,  the circle group. By \eqref{eq:bu1}, $BG=B\SO(2)\simeq \mathbb{C}P^{\infty}$ and so
\begin{equation}\label{eq:isom}
\check{H}^1\bigl(K_\varepsilon,\SO(2)\bigr)\simeq [K_\varepsilon, B\SO(2)] \simeq [K_\varepsilon, \mathbb{C}P^{\infty}] \simeq H^2(K_\varepsilon),
\end{equation}
where the last isomorphism is by Theorem~\ref{thm:K(Z,2) space}. We will discuss the two main implications of \eqref{eq:isom} separately:  $H^2(K_\varepsilon)$ gives us a homological classification of continuous cryo-EM cocycles; whereas $[K_\varepsilon, B\SO(2)]$ tells us about the \textit{moduli} space of continuous cryo-EM cocycles.

\subsection{Cohomology as obstruction}

The cohomology group $H^2(K_\varepsilon)$ may be viewed as the obstruction to $K_\varepsilon$ degenerating into a one-dimensional simplicial complex. If $H^2(K_\varepsilon) = 0$, then $K_\varepsilon$ contains no $2$-sphere --- by which we mean the boundary of a $3$-simplex, which is homeomorphic to $\mathbb{S}^2$. Thus  $K_\varepsilon$ is  a two-dimensional simplicial complex whose $2$-simplices are all contractible, and thus it is homotopic to a one-dimensional simplicial complex. Let $H^2(K_\varepsilon) = 0$. If $\psi_j$, $\psi_k$, $\psi_l$ are  three images that lie in the $\varepsilon$-neighborhood of an image $\psi_i$, then at least one of $\psi_j$, $\psi_k$, $\psi_l$ cannot lie in the intersection of $\varepsilon$-neighborhoods of the other two. In terms of the graph $G_\varepsilon$, $H^2(K_\varepsilon) = 0$ implies that $G_\varepsilon$ does not contain a $4$-clique, i.e., a complete subgraph with four vertices.

The isomorphism wtih $H^2(K_\varepsilon)$ also allows us to calculate $\check{H}^1\bigl(K_\varepsilon,\SO(2)\bigr)$ explicitly. 
\begin{theorem}\label{thm:cohomology of K}
$\check{H}^1\bigl(K_\varepsilon,\SO(2)\bigr) \simeq H^2(K_\varepsilon) = \mathbb{Z}^{b}$ where
$b = b_2(K_\varepsilon)$, the second Betti number of $K_\varepsilon$.
\end{theorem}
\begin{proof}
The isomorphism is \eqref{eq:isom}. The equality follows from Theorem~\ref{thm:universal coefficient}, observing that $H_1(K_\varepsilon)=0$ by our construction of $K_\varepsilon$ and so $T_1 = 0$. We may also derive the isomorphism  directly without going through the chain of isomorphisms in \eqref{eq:isom}. Snake Lemma \cite{Hatcher, May, Spanier} applied to the exact sequence of groups 
\[
1 \to \mathbb{Z} \xrightarrow{2\pi } \mathbb{R} \xrightarrow{\operatorname{expi}} \mathbb{S}^{1}\to 1,
\]
where the first map is multiplication by $2\pi$ and $\operatorname{expi}(x) \coloneqq \exp(ix)$, yields a long exact sequence of cohomology groups
\[
\dots \to \check{H}^1(K_\varepsilon,\mathbb{R}) \to \check{H}^{1}(K_\varepsilon,\mathbb{S}^{1}) \to \check{H}^2(K_\varepsilon,\mathbb{Z})\to \check{H}^2(K_\varepsilon,\mathbb{R})\to \cdots 
\]
Both $\check{H}^1(K_\varepsilon,\mathbb{R})$ and $\check{H}^2(K_\varepsilon,\mathbb{R})$ are zero by the existence of partition of unity on $K_\varepsilon$. So $\check{H}^1(K_\varepsilon,\mathbb{S}^{1}) = \check{H}^2(K_\varepsilon,\mathbb{Z})$. Since $\mathbb{S}^{1} =\SO(2)$, $\check{H}^1(K_\varepsilon,\mathbb{S}^{1}) =\check{H}^1\bigl(K_\varepsilon,\SO(2)\bigr)$. Finally, by Theorem~\ref{thm:singular=cech}, we get $\check{H}^2(K_\varepsilon,\mathbb{Z}) \simeq H^2(K_\varepsilon,\mathbb{Z})=H^2(K_\varepsilon)$.
\end{proof}

\subsection{Cohomology as moduli}

A benefit of classifying continuous cryo-EM cocycles in terms of oriented circle bundles is that these are very well understood classical objects \cite{Chern, Steenrod}. In what follows, we will refine Theorem~\ref{thm:class} with explicit descriptions of the oriented circle bundles that arise in the classification of continuous cryo-EM cocycles.

Let $b_2(K_\varepsilon) = b$. Since $K_\varepsilon$ is a finite two-dimensional simplicial complex, $b_2(K_\varepsilon) = b$ implies that $K_\varepsilon$ contains $b$ copies of $2$-spheres.  By \eqref{eq:bunclass} and Theorem~\ref{thm:cohomology of K}, we expect to obtain an oriented circle bundle over $K_\varepsilon$ for each $(m_1,\dots,m_b)\in \mathbb{Z}^{b}$. An oriented circle bundle over any one-dimensional simplicial complex $K$ must be trivial since $H^2(K) = 0$. Hence any oriented circle bundle over $K_\varepsilon$ is uniquely determined by its restriction to the $2$-spheres contained in $K_\varepsilon$ and understanding oriented circle bundles on $K_\varepsilon$ reduces to understanding oriented circle bundles on $\mathbb{S}^2$, which we will describe explicitly in the following.

We start by identifying the $3$-sphere with the group of unit quaternions, i.e.,
\[
\mathbb{S}^3 = \{a + bi + cj +dk\in \mathbb{H} : a,b,c,d\in \mathbb{R}, \; a^2 + b^2 + c^2 + d^2 =1\},
\]
and identify the circle  with the group of unit complex numbers, i.e.,
\[
\mathbb{S}^1 = \{a + bi \in \mathbb{C} : a, b \in \mathbb{R},\; a^2 + b^2 = 1 \}.
\]
Elements of $\mathbb{S}^1$ may be regarded as unit quaternions with $c = d = 0$ and so $\mathbb{S}^1$ a subgroup of $\mathbb{S}^3$. In particular, $\mathbb{S}^1$ acts on $\mathbb{S}^3$ by quaternion multiplication and we have a group action 
\begin{equation}\label{eq:group}
\varphi : \mathbb{S}^1 \times \mathbb{S}^3 \to \mathbb{S}^3, \quad (x + yi, a + bi +cj +dk) \mapsto xa -yb + (xb + ya)i + (xc - yd)j + (xd + yc)k.
\end{equation}
As topological spaces we have 
\[
\mathbb{S}^3/\mathbb{S}^1 \simeq \mathbb{S}^2
\]
but note that $\mathbb{S}^1$ is not a normal subgroup of $\mathbb{S}^3$ and so $\mathbb{S}^2$ does not inherit a group structure. Let
\[
\pi : \mathbb{S}^3 \to \mathbb{S}^3/\mathbb{S}^1 \simeq \mathbb{S}^2
\]
be the natural quotient map.

For $m \in \mathbb{N}$, let $C_m$ be the subgroup of $\mathbb{S}^1$ generated by $\exp (2\pi i/m)$, a cyclic group of order $m$. Each $C_m$  is also a subgroup of $\mathbb{S}^3$ and acts on $\mathbb{S}^3$ by quaternion multiplication.
Since $C_m$ is a subgroup of $\mathbb{S}^1$, we obtain an induced projection map
\begin{equation}\label{eq:proj}
\pi_m: \mathbb{S}^3/C_m \to \mathbb{S}^3/\mathbb{S}^1\simeq \mathbb{S}^2
\end{equation}
for each $m \in \mathbb{N}$.
The following classic result \cite{Steenrod} describes all \textit{circle bundles} on $\mathbb{S}^2$ --- there are infinitely many of them, one for each nonnegative integer. 
\begin{proposition}
For each $m=0,1,2,\dots$, there is a circle bundle $(A_m, \pi_m, \varphi_m)$ with base space $\mathbb{S}^2$ where
\[
A_0 = \mathbb{S}^1 \times \mathbb{S}^2, \qquad A_m = \mathbb{S}^3/C_m \quad \text{for} \; m \in \mathbb{N}.
\]
The projection to $\mathbb{S}^2$,
\[
\pi_0 : A_0 \to \mathbb{S}^2, \qquad \pi_m  : A_m \to \mathbb{S}^3/\mathbb{S}^1\simeq \mathbb{S}^2,
\]
is the projection onto the second factor for $m = 0$ and the quotient map \eqref{eq:proj} for $m\in \mathbb{N}$. The group action $\varphi_m : \mathbb{S}^1 \times A_m \to A_m$ is  the trivial action (any element in $\mathbb{S}^1$ acts as identity on $A_0$) for $m = 0$ and the action induced by quaternion multiplication $\varphi$ in \eqref{eq:group}  for $m\in \mathbb{N}$. Every circle bundle on $\mathbb{S}^2$ is isomorphic to an $A_m$ for some $m = 0, 1,2,\dots.$
\end{proposition}
Note that these are $\SO(2)$-bundles since we regard $\SO(2)  = \mathbb{S}^1$. $A_0$ is the trivial circle bundle on $\mathbb{S}^2$ and $A_1$ is the well-known \textit{Hopf fibration}. As a manifold, $A_m = \mathbb{S}^3/C_m$ is orientable for all  $m \in \mathbb{N}$ and so each $A_m$ comes in two different orientations, which we denote by $A_m^+$ and $A_m^-$.  For $m =0,1,2,\dots,$ we write
\[
B_0 \coloneqq A_0, \qquad B_m \coloneqq A_m^+, \qquad B_{-m} \coloneqq A_m^-.
\]
These  are  the \textit{oriented circle bundles} on $\mathbb{S}^2$. 

We will next construct a \textit{cryo-EM bundle} by gluing oriented circle bundles along the cryo-EM complex $K_\varepsilon$, attaching a copy of $B_m$ for some $m \in \mathbb{Z}$ to each $2$-sphere in $K_\varepsilon$. We then show that these bundles are in one-to-one correspondence with continuous cryo-EM cocycles on $K_\varepsilon$.

Let $K_\varepsilon$ be a cryo-EM complex with $b_2(K_\varepsilon) = b$, i.e., $K_\varepsilon$ contains $b$ copies of $2$-spheres; in fact, by its definition, $K_\varepsilon$ is homotopic to the one-point union of $b$ copies of $\mathbb{S}^2$, as we discussed in the proof of Theorem~\ref{thm:flat=trivial}. Label these arbitrarily from $i=1,\dots,b$ and denote them $\mathbb{S}^2_1,\dots,\mathbb{S}^2_b$. For any $(m_1,\dots, m_b) \in \mathbb{Z}^b$, we may define a principal $\SO(2)$-bundle $B_{m_1,\dots,m_b}$ on $K_\varepsilon$ as one whose restriction on the $i$th $2$-sphere in $K_\varepsilon$ is $B_{m_i}$, $i=1,\dots,b$, and is trivial elsewhere. 
We remove all the $2$-spheres contained in $K_\varepsilon$ and let the remaining simplicial complex be
\[
L_\varepsilon\coloneqq \Bigl( \overline{K_\varepsilon - \bigcup\nolimits_{i=1}^b\mathbb{S}^2_i}\Bigr).
\]
As a topological space, $B_{m_1,\dots,m_b}$  is the union of $B_{m_i}$'s corresponding to each of the $2$-spheres and  the trivial circle bundle on $L_\varepsilon$,
\[
B_{m_1,\dots,m_b} \coloneqq \Bigl[\bigcup\nolimits_{i=1}^b B_{m_i}\Bigr]\cup \Bigl[L_\varepsilon \times \mathbb{S}^1\Bigr].
\]
To see that $B_{m_1,\dots,m_b}$ is a fiber bundle on $K_\varepsilon$, take the open covering
\[
\mathcal{U} = \{U_1(K_\varepsilon),\dots, U_{n}(K_\varepsilon)\}
\]
of $K_\varepsilon$  in Section~\ref{sec:main}. By the construction of $B_{m_1,\dots, m_b}$, its restriction  to $U_i(K_\varepsilon)$ is a trivial fiber bundle since  $U_i(K_\varepsilon)$ is contractible. So $B_{m_1,\dots,m_b}$ is locally trivial and thus a fiber bundle on $K_\varepsilon$. Moreover, the bundle $(B_{m_1,\dots,m_b},\pi,\varphi)$ is an oriented circle bundle on $K_\varepsilon$ with $\pi$ and $\varphi$ defined as follows. The projection map $\pi : B_{m_1,\dots, m_b} \to K_\varepsilon$ is defined by
\[
\pi(f) = \begin{cases}
\pi_{m_{i}} (f), &\text{if}\; f\in B_{m_i},\quad i=1,\dots,b,\\
\operatorname{pr}_1(f), &\text{if}\; f\in L_\varepsilon \times \mathbb{S}^1.
\end{cases}
\] 
Here $\operatorname{pr}_1: L_\varepsilon\times \mathbb{S}^1 \to  L_\varepsilon$ is the projection onto the first factor. The group action $\varphi: \SO(2) \times B_{m_1,\dots,m_b} \to B_{m_1,\dots,m_b}$ is defined by
\[
\varphi(g,f) =
\begin{cases}
\varphi_{m_i}(g,f), &\text{if}\; f\in B_{m_i},\quad i=1,\dots,b,\\
f, &\text{if}\; f\in  L_\varepsilon\times \mathbb{S}^1,
\end{cases}
\]
for any $g\in G$ and $f\in B_{m_1,\dots,m_b}$. Furthermore, the intersection of any two simplices in $K_\varepsilon$ is by our construction either empty or a contractible space and so any bundle is trivial on the intersection.

Since every oriented circle bundle on $K_\varepsilon$ is isomorphic to $B_{m_1,\dots,m_b}$ for some  $(m_1,\dots, m_b) \in \mathbb{Z}^b$, we have the following classification theorem for continuous cryo-EM cocycles in terms of $B_{m_1,\dots,m_b}$.
\begin{theorem}[Bundle classification of continuous cryo-EM cocycles II]\label{thm:explicit construction}
Let $\varepsilon > 0 $ be small enough so that \eqref{eqn:cocycle tau_ij} holds and let $K_\varepsilon$ be the corresponding cryo-EM complex. Let $b = b_2(K_\varepsilon)$. Then each cohomologically equivalent continuous cryo-EM cocycles $z^c_\varepsilon$ on $K_\varepsilon$ corresponds to an isomorphism class of an oriented circle bundle $B_{m_1,\dots,m_b}$ on $K_\varepsilon$ for $(m_1,\dots, m_b) \in \mathbb{Z}^b$.
\end{theorem}
\begin{proof}
Let $z^c_\varepsilon = \{g_{ij} \in \uline{\SO(2)}(U_i\cap U_j): [i,j]\in K_\varepsilon\}$ and $z^{\prime c}_\varepsilon = \{g'_{ij} \in \uline{\SO(2)}(U_i\cap U_j): [i,j]\in K_\varepsilon\}$ be cohomologically equivalent continuous cryo-EM cocycles on $K_\varepsilon$, i.e., they are related by \eqref{eq:cross1'} for some $g_i, g_j\in \SO(2)$, $[i,j] \in K_\varepsilon$.
By Theorem~\ref{thm:class'}, $z^c_\varepsilon$ and $z^{\prime c}_\varepsilon$ must correspond to the same oriented circle bundle on $K_\varepsilon$. 
\end{proof}

\section{Denoising cryo-EM images and cohomology}\label{sec:denoise}

Aside from providing theoretical classification results (e.g.\ Theorems~\ref{thm:class} and \ref{thm:explicit construction}) whose practical value is as yet unclear, we show here that the more elementary aspects of our cohomology framework can shed light on one aspect of cryo-EM imaging --- denoising cryo-EM images. Our goal is not to propose any new method  but to provide some perspectives on existing methods, which work well in practice \cite{SW1,SZSH,Singer}. We saw in  Section~\ref{sec:main} that a \textit{noiseless} discrete cryo-EM cocycle $z^d_\varepsilon =\{ g_{ij}\in \SO(2): [i,j]\in K_\varepsilon\}$ on $K_\varepsilon$ satisfies the cocycle condition
\begin{equation}\label{eq:cocycle-repeat}
g_{ij} g_{jk} g_{ki} =1,
\end{equation}
when $\varepsilon$ is sufficiently small. In reality, a collection of projected images obtained from cryo-EM measurements, $\widehat{\Psi} =\{\widehat{\psi}_1,\dots,\widehat{\psi}_n\} $,  will invariably be corrupted by noise; here a carat over a quantity signifies that it is possibly corrupted by noise. As a result, the discrete cryo-EM cocycle $\widehat{z}_\varepsilon =\{ \widehat{g}_{ij}\in \SO(2): [i,j]\in K_\varepsilon\}$ obtained from $\widehat{\Psi}$  will not satisfy \eqref{eq:cocycle-repeat} for sufficiently small $\varepsilon > 0$. To see this, let $\Psi = \{\psi_1,\dots, \psi_n\}$ be a set of noise-free projected images whose discrete cryo-EM cocycle is $z^d_\varepsilon$. If $\widehat{z}_\varepsilon$ also satisfies \eqref{eq:cocycle-repeat}, then we have 
\[
\widehat{g}_{ij} = g_{ij} g_i g_j^{-1} \in \SO(2)
\]
for some $g_i\in \SO(2)$. But this implies that $\widehat{\psi_i}$ can be obtained by rotating the noise-free image $\psi_i$ by $g_i\in \SO(2)$ and so $\widehat{\psi}_i$ is also noise-free, a contradiction.

In general cryo-EM images are denoised by \textit{class averaging} \cite{Frank2}. Noisy images are grouped into classes of similar viewing directions. The within-class average is  then taken as an approximation of the noise-free image in that direction. The methods for grouping images into classes \cite{SW1,SZSH,Singer} are essentially all based on the observation that in the noiseless scenario, the cocycle condition \eqref{eq:cocycle-repeat} must hold. We will look at a few measures of deviation of discrete cryo-EM cocycles from being a cocycle.

Let $\widehat{z}^d_\varepsilon =\{ \widehat{g}_{ij}\in \SO(2): [i,j]\in K_\varepsilon\}$  be a discrete cryo-EM cocycle on $K_\varepsilon$ computed from a noisy set of projected images $\widehat{\Psi} $.  Since $\SO(2)$ can be identified with the circle $\mathbb{S}^1$, every $g\in \SO(2)$ corresponds to an angle $\theta \in \mathbb{S}^1$, represented by $\theta \in [0,2\pi)$. A straightforward measure of deviation of $z_\varepsilon$ from being a cocycle is given by
\[
\delta(\widehat{z}^d_\varepsilon) =  \sum\nolimits_{i, j,k \; : \; [i,j],[i,k],[j,k] \in K_\varepsilon} (\theta_{ij} + \theta_{jk} +  \theta_{ki})^2, 
\]
where the addition in the parentheses is computed in $\mathbb{S}^1$, i.e., given by the unique number $\theta_{ijk} \in [0, 2\pi)$ such that
\[
\theta_{ij} + \theta_{jk} +  \theta_{ki} = \theta_{ijk} \pmod{2\pi}.
\]
\begin{lemma}
$\widehat{z}^d_\varepsilon$ is a cocycle if and only if $\delta(\widehat{z}^d_\varepsilon) =0$. 
\end{lemma}
Let $\psi$ be an arbitrary projected image. Then $\delta(\widehat{z}_\varepsilon) $ quantifies the obstruction of gluing images in
\[
\widehat{z}^d_\varepsilon(\psi) = \{ g  \cdot \psi : g \in \widehat{z}^d_\varepsilon\} = \{\widehat{g}_{ij} \cdot \psi : [i,j] \in K_\varepsilon\}
\]
together to get the $3$D structure of the molecule. If $\delta(\widehat{z}_\varepsilon) $ is small, then $\widehat{z}^d_\varepsilon$ is already close enough to a cocycle and hence every image is good.

On the other hand, if $\delta(\widehat{z}^d_\varepsilon) $ is big, then the following measure allows us to identify subsets of good images, if any. Given an image $\psi$, whenever $[i,j]$ is an edge of $K_\varepsilon$ for some $j$, we want the viewing direction of $g_{ij} \cdot \psi$ to be close to that of $\psi$.
This is captured by the quantity $\rho_i (\widehat{z}^d_\varepsilon) \coloneqq \delta_i(\widehat{z}^d_\varepsilon) /3\delta(\widehat{z}^d_\varepsilon) $ where
\[
\delta_i(\widehat{z}^d_\varepsilon)  = \sum\nolimits_{j,k \; : \; [i,j],[i,k],[j,k] \in K_\varepsilon} (\theta_{ij} + \theta_{jk} +  \theta_{ki})^2, \qquad i=1, \dots, n.
\]
Clearly, $\sum_{i=1}^n \rho_i (\widehat{z}^d_\varepsilon)  = 1$. For $g_{ij} \cdot \psi$ to be a good image, we want $\rho_i (\widehat{z}^d_\varepsilon) \ll 1 $.

\section*{Acknowledgment}

We thank Amit Singer for introducing us to this fascinating topic. We are extremely grateful to the two anonymous referees for their unusually thorough reading  and for making numerous useful comments that led to a vast improvement of our manuscript.
The work in this article is generously supported by AFOSR FA9550-13-1-0133, DARPA D15AP00109, NSF IIS 1546413, DMS 1209136, DMS 1057064.


\begin{thebibliography}{99}    
\bibitem{AK1} G.~Ambartsoumian and P.~Kuchment, ``A range description for the planar circular Radon transform," \textit{SIAM J.\ Math.\ Anal.}, \textbf{38} (2006), no.~2, pp.~681--692.
 
\bibitem{AK} G.~Ambartsoumian and P.~Kuchment, ``On the injectivity of the circular Radon transform arising in thermoacoustic tomography," \textit{Inverse Probl.}, \textbf{21} (2005), no.~2, pp.~473–-485.

\bibitem{medical1} S.~R.~Arridge, ``Optical tomography in medical imaging", \textit{Inverse Probl.}, \textbf{15} (1999), no.~2, pp.~R41--R93.

\bibitem{medical2} S.~R.~Arridge and J.~C.~Hebbden, ``Optical imaging in medicine: II. Modelling and reconstruction," \textit{Phys.\ Med.\ Biol.}, \textbf{42} (1997), no.~42, pp.~841.

\bibitem{remote2} C.~Baillard and H.~Maitre, ``3-D reconstruction of urban scenes from aerial stereo imagery: A Focusing Strategy," \textit{Comput.\ Vis.\  Image.\ Und.}, \textbf{76} (1999), no.~3, pp.~244--258.

\bibitem{BCS} A.~S.~Bandeira, Y.~Chen, and A.~Singer, ``Non-unique games over compact groups and orientation estimation in cryo-EM,''  \url{http://arxiv.org/abs/1505.03840}, \textit{preprint}, 2015.

\bibitem{BKS} A.~S.~Bandeira, C.~Kennedy, and A.~Singer, ``Approximating the little Grothendieck problem over the orthogonal and unitary groups,'' \textit{Math.\ Program.}, to appear, (2016).

\bibitem{B} C.~Berge, \textit{Hypergraphs}, North-Holland Mathematical Library, \textbf{45}, North-Holland,  Amsterdam, 1989.

\bibitem{Brown} E.~H.~Brown, ``Cohomology theories,'' \textit{Ann.\ of Math.}, \textbf{75} (1962), pp.~467–-484.

\bibitem{Brylinski} J.-L.~Brylinski, \emph{Loop Spaces, Characteristic Classes and Geometric Quantization}, Birkh\"auser, Boston, MA, 2008.

\bibitem{underwater1} U.~Castellani, A.~Fusiello, and V.~ Murino, ``Registration of multiple acoustic range views for underwater scene reconstruction," \textit{Comput.\ Vis.\ Image.\ Und.}, \textbf{87} (2002), no.~1, pp.~78--89.
  
\bibitem{Chern} S.~S.~Chern, ``Circle bundles,'' pp.~114--131, \textit{Geometry and Topology}, Lecture Notes in Mathematics, \textbf{597}, Springer, Berlin, 1977. 

\bibitem{DSG} V.~De Silva and R.~Ghrist, ``Homological sensor networks,'' \textit{Notices Amer.\ Math.\ Soc.}, \textbf{54} (2007), no.~1, pp.~10--17. 

\bibitem{FPR} D.~Finch, S.~K.~Patch, and Rakesh, ``Determining a function from its mean values over a family of spheres," \textit{SIAM J.\ Math.\ Anal.}, \textbf{35} (2004), no.~5, pp.~1213–-1240.

\bibitem{remote1} A.~Fischer, T.~H.~Kolbe, F.~Lang, A.~B.~Cremers, W.~F\"orstner, L.~Pl\"umer, and V.~Steinhage, ``Extracting buildings from aerial images using hierarchical aggregation in 2D and 3D," \textit{Comput.\ Vis.\  Image.\ Und.}, \textbf{72} (1998), no.~1, pp.~185--203.

\bibitem{Forster} O.~Forster, \emph{Lectures on {R}iemann surfaces}, Graduate Texts in Mathematics, \textbf{81}, Springer, New York, NY, 1991.

\bibitem{Francis} G.~K.~Francis, \emph{A Topological Picturebook}, Springer, New York, NY, 2007.

\bibitem{Frank1} J.~Frank, ``Single-particle imaging of macromolecules by cryo-electron microscopy,'' \textit{Annu.\ Rev.\ Biophys.\ Biomol.\ Struct.}, \textbf{31} (2002), pp.~303–319.

\bibitem{Frank2} J.~Frank, \textit{Three-Dimensional Electron Microscopy of Macromolecular Assemblies: Visualization of biological molecules in their native state}, Oxford University Press, New York, NY, 2006.

\bibitem{Godment} R.~Godement, \emph{Topologie Alg\'ebrique et Th\'eorie des Faisceaux}, Troisi{\`e}me {\'e}dition revue et corrig{\'e}e, Publications de l'Institut de Math{\'e}matique de l'Universit{\'e} de Strasbourg, XIII, Actualit{\'e}s Scientifiques et Industrielles, \textbf{1252}, Hermann, Paris, 1973.

\bibitem{GH} P.~Griffiths and J.~Harris, \textit{Principles of Algebraic Geometry}, Wiley, New York, NY, 1994.      
 
\bibitem{HS1} R.~Hadani and A.~Singer, ``Representation theoretic patterns in three-dimensional cryo-electron microscopy {I}: {T}he intrinsic reconstitution algorithm,'' \textit{Ann.\ of Math.}, \textbf{174} (2011), no.~2, pp.~1219--1241.

\bibitem{HS2} R.~Hadani and A.~Singer, ``Representation theoretic patterns in three-dimensional cryo-electron microscopy {II}: {T}he class averaging problem,'' \textit{ Found.\ Comp.\ Math.},\textbf{11} (2011), no.~5, pp.~589--616.

\bibitem{Hartshorne} R.~Hartshorne, \textit{Algebraic Geometry}, Graduate Texts in Mathematics, \textbf{52}, Springer, New York, NY, 1977.

\bibitem{Hatcher} A.~Hatcher, \textit{Algebraic Topology}, Cambridge University Press, Cambridge, UK, 2002.

\bibitem{Hungerford} T.~W.~Hungerford, \textit{Algebra}, Graduate Texts in Mathematics, \textbf{73}, Springer, New York, NY, 1980.

\bibitem{Husemoller} D.~Husem{\"o}ller, \emph{Fibre Bundles}, 3rd Ed., Graduate Texts in Mathematics, \textbf{20}, Springer, New York, 1994.

\bibitem{HJJS} D.~Husem{\"o}ller, M.~Joachim, B.~Jur{\v{c}}o, and M.~Schottenloher, \emph{Basic Bundle Theory and {$K$}-Cohomology Invariants}, Lecture Notes in Physics, \textbf{726}, Springer, Berlin, 2008.

\bibitem{Iversen} B.~Iversen, \textit{Cohomology of Sheaves}, Universitext, Springer, Berlin, 1986.

\bibitem{John} F.~John, \textit{Plane Waves and Spherical Means Applied to Partial Differential Equations}, Interscience Publishers, New York, NY, 1955.

\bibitem{KT} F.~Kamber and Ph.~Tondeur, ``Flat bundles and characteristic classes of group-representations," \textit{Amer.\ J.\ Math.}, \textbf{89} (1967), no.~4, pp.~857--886.

\bibitem{KLFA} R.~A.~Kruger, P.~Liu, Y.~R.~Fang, and C.~R.~Appledorn, ``Photoacoustic ultrasound (\textsc{paus}) reconstruction tomography," \textit{Med.\ Phys.}, \textbf{22} (1995), no.~10, pp.~1605–-1609.

\bibitem{L} L.-H.~Lim, ``Hodge Laplacians on graphs,'' S.~Mukherjee (Ed.), \textit{Geometry and Topology in Statistical Inference}, Proceedings of Symposia in Applied Mathematics, \textbf{73}, AMS, Providence, RI, 2015.

\bibitem{May} J.~P.~May, \textit{A Concise Course in Algebraic Topology}, Chicago Lectures in Mathematics, University of Chicago Press, Chicago, IL, 1999.

\bibitem{MS} J.~W.~Milnor and J.~D.~Stasheff, \textit{Characteristic Classes}, Princeton University Press, Princeton, NJ, 1974.

\bibitem{Morita} S.~Morita, \textit{Geometry of Characteristic Classes},  Translations of Mathematical Monographs, \textbf{199}, AMS, Providence, RI, 2001.

\bibitem{Natterer} F.~Natterer, \textit{The Mathematics of Computerized Tomography}, Classics in Applied Mathematics,  SIAM, Philadelphia, PA, 2001.

\bibitem{underwater2} S.~Negahdaripour, H.~Sekkati, and H.~Pirsiavash, ``Opti-acoustic stereo imaging: On system calibration and 3-D target reconstruction," \textit{IEEE Trans.\ Image Process.}, \textbf{18} (2009), no.~6, pp.~1203--1214.

\bibitem{OT} J.~Oprea and D.~Tanr\'e, ``Flat circle bundles, pullbacks, and the circle made discrete,'' \textit{Int.\ J.\ Math.\ Math.\ Sci.}, (2005), no.~21, pp.~3487--3495.

\bibitem{PZF} P.~A.~Penczek, J.~ Zhu, and J.~Frank, ``A common-lines based method for determining orientations for $N>3$ particle projections simultaneously,'' \textit{Ultramicroscopy}, \textbf{63 }(1996), nos.~3--4, pp.~205--218.

\bibitem{PRF} P.~Penczek, M. Radermacher, and J.~Frank, ``Three-dimensional reconstruction of single particles embedded in ice,'' \textit{Ultramicroscopy}, \textbf{40} (1992), pp.~33--53.

\bibitem{Penrose} R.~Penrose, ``On the cohomology of impossible figures,'' \textit{Structural Topology}, \textbf{17} (1991), pp.~11--16.

\bibitem{Prasolov} V.~V.~Prasolov, \textit{Elements of Homology Theory}, Graduate Studies in Mathematics, \textbf{81}, AMS, Providence, RI, 2007.      

\bibitem{Radermacher} M.~Radermacher, ``Three-dimensional reconstruction from random projections: orientation alignment via Radon transforms,'' \textit{Ultramicroscopy}, \textbf{53} (1994), pp.~121--136.

\bibitem{medical3} M.~Schweiger and S.~R.~Arridge, ``Optical tomographic reconstruction in a complex head model using a priori region boundary information," \textit{Phys.\ Med.\ Biol.}, \textbf{44} (1999), no.~11, pp.~2703.

\bibitem{Singer} A.~Singer, ``Angular synchronization by eigenvectors and semidefinite programming," \textit{Appl.\ Comput.\ Harmon. Anal.}, \textbf{30} (2011), no.~1, pp.~20--36.

\bibitem{SS} A.~Singer and Y.~Shkolnisky, ``Three-dimensional structure determination from common lines in cryo-EM by eigenvectors and semidefinite programming,'' \textit{SIAM J.\ Imaging Sci.}, \textbf{4} (2011), no.~2, pp.~543--572.

\bibitem{SW2} A.~Singer and H.-T.~Wu, ``Two-dimensional tomography from noisy projections taken at unknown random directions,''  \textit{SIAM J.\ Imaging Sci.}, \textbf{6} (2013), no.~1, pp. 136--175.

\bibitem{SW1} A.~Singer and H.-T.~Wu, ``Vector diffusion maps and the connection laplacian,'' \textit{Comm.\ Pure Appl.\ Math.}, \textbf{65} (2012), no.~8, pp.~1067--1144.

\bibitem{SZSH} A.~Singer, Z.~Zhao, Y.~Shkolnisky, and R.~Hadani, ``Viewing angle classification of cryo-electron microscopy images using eigenvector,''  \textit{SIAM J.\ Imaging Sci.}, \textbf{4} (2011), no.~2, pp.~723--759.

\bibitem{Spanier} E.~H.~Spanier, \textit{Algebraic Topology}, Springer, New York, NY, 1981.

\bibitem{Steenrod} N.~E.~Steenrod, ``The classification of sphere bundles,'' \textit{Ann.\ Math.}, \textbf{45} (1944), no.~2, pp.~294--311.

\bibitem{VG} B.~Vainshtein and A.~Goncharov, ``Determination of the spatial orientation of arbitrarily arranged identical particles of an unknown structure from their projections,'' \textit{Proc.\ Int.\  Congress Electron Microscopy}, \textbf{11} (1986), pp.~459--460.

\bibitem{Heel} M.~Van~Heel,``Angular reconstitution: A posteriori assignment of projection directions for 3D reconstruction,'' \textit{Ultramicroscopy}, \textbf{21} (1987), no.~2, pp.~111--123.   

\bibitem{Rips} L.~Vietoris, ``\"{U}ber den h\"oheren {Z}usammenhang kompakter {R}\"aume und eine {K}lasse von zusammenhangstreuen {A}bbildungen,'' \textit{Math. Ann.}, \textbf{97} (1927), no.~1, pp.~454--472.

\bibitem{XW} M.~Xu and L.-H.~V.~Wang, ``Time-domain reconstruction for thermoacoustic tomography in a spherical geometry," \textit{IEEE Trans.\ Med.\ Imag.}, \textbf{21} (2002), no.~7, pp.~814-–822.
\end{thebibliography}
\end{document}